\pdfoutput=1

\documentclass[twoside]{article}
\usepackage[accepted]{aistats2020}
\usepackage{amsthm}
\usepackage{bm}
\usepackage[dvips]{graphics}
\usepackage{graphicx}
\usepackage[colorlinks=false,allcolors=blue]{hyperref}
\usepackage{booktabs}

\usepackage{microtype}
\usepackage{mathtools}

\usepackage{psfrag}
\usepackage{units}
\usepackage{sidecap}
\usepackage{amsfonts,amsmath,amssymb}

\hypersetup{colorlinks=true}

\usepackage{dsfont,color,setspace, cite}
\usepackage{caption,cite}
\usepackage{subcaption}
\usepackage{wrapfig}
\usepackage{comment}

\DeclareMathOperator*{\argmin}{arg\,min}

\usepackage{xspace}
\usepackage{bbm}
\def\diag{\mathop{\rm diag}\nolimits}%

\DeclareMathOperator{\Var}{Var}
\DeclareMathOperator{\var}{Var}

\newcommand{\xv}{{\bm x}}
\newcommand{\yv}{{\bm y}}

\newcommand{\betav}{{\bm \beta}}

\def\e{\epsilon}

\DeclareMathOperator\E{\mathbb{E}}
\let\P\relax
\DeclareMathOperator\P{P}
\DeclareMathOperator\R{R}

\newcommand{\N}{\mathrm{N}}

\def\textiid{i.i.d.\@\xspace}
\newcommand\iid{\ifmmode\text{ i.i.d. } \else \textiid \fi}

\DeclarePairedDelimiter{\abs}{\lvert}{\rvert}
\DeclarePairedDelimiter{\norm}{\lVert}{\rVert}

\newtheorem{assumption}{Assumption}
\newtheorem{theorem}{Theorem}
\newtheorem{lemma}{Lemma}

\newtheorem{corollary}{Corollary}
\newtheorem{remark}{Remark}

\newtheorem{example}{Example}

\makeatletter
\newcommand{\neutralize}[1]{\expandafter\let\csname c@#1\endcsname\count@}
\makeatother

\newenvironment{assumptionbis}[1]
  {\renewcommand{\theassumption}{\ref*{#1}$'$}%
   \neutralize{assumption}\phantomsection
   \begin{assumption}}
  {\end{assumption}}

\begin{document}

%\long\def\/*#1*/{}

%\title{\LARGE \bf A scalable estimate of the extra-sample prediction error  via approximate leave-one-out}
%\date{}
%
%\author{Kamiar Rahnama Rad and Arian Maleki
%% <-this % stops a space
%%\thanks{This work was not supported by any organization}% <-this % stops a space
%%\thanks{...}
%}

\runningtitle{Error bounds in estimating the out-of-sample prediction error using LOOCV}
\newcommand{\blind}{1}  % BLIND
\newcommand{\longer}{1} % LARGE N SMALL P

\newcommand{\lo}{{\rm LO}}
\newcommand{\extra}{{\rm Err}_{\rm out}}
\newcommand{\ine}{\rm{Err}_{\rm in}}

% our title is too long for the running title, we should choose a shorter one.
\twocolumn[
    \aistatstitle{Error bounds in estimating the out-of-sample prediction error using leave-one-out cross validation in high-dimensions}
    \aistatsauthor{Kamiar Rahnama Rad \And Wenda Zhou \And Arian Maleki}
    \aistatsaddress{Baruch College \\ City University of New York \And Columbia University \And Columbia University}
]
%\if1\blind
%{
%  \title{\bf Error bounds in estimating the out-of-sample prediction %error using leave-one-out cross validation in high-dimensions  }
%  \author{Kamiar Rahnama Rad\thanks{Baruch College, City University of New York}, Wenda Zhou\thanks{Department of Statistics, Columbia University}, Arian Maleki$^\dagger$}
%    \date{}
%  \maketitle
%} \fi

%\if0\blind
%{
%  \bigskip
%  \bigskip
%  \bigskip
%  \begin{center}
%    {\LARGE\bf Error bounds in estimating the out-of-sample prediction error using leave-one-out cross validation in high-dimensions }
%\end{center}
%  \medskip
%} \fi

%\maketitle

\newcommand{\io}{\underline{1}}

\newcommand{\alo}{{\rm ALO}}

\newcommand{\inew}{\widehat{{\rm Err}}_{\rm in}}
\newcommand{\p}{\mathds{P}}
\newcommand{\mb}{\mathbf{m}}
\newcommand{\bb}{\mathbf{b}}
\newcommand{\bl}{\bm{\hat \beta}}
\newcommand{\blo}{\hat \beta^\circ}

\newcommand{\btjpi}{\bm{ \hat \beta}_{\ \tilde i/j' }}
\newcommand{\bljp}{\bm{ \hat \beta}_{/j'}}

\newcommand{\bli}{\bm{ \hat \beta}_{/i }}
\newcommand{\blijj}{\bm{ \hat \beta}_{/ijj' }}

\newcommand{\blj}{\bm{ \hat \beta}_{/j}}
\newcommand{\bti}{\bm{ \hat \beta}_{ /i }}

\newcommand{\btji}{\bm{ \hat \beta}_{\ \tilde i /ij }}

\newcommand{\blone}{\bm{\hat \beta}_{/1}}
\newcommand{\bltwo}{\bm{\hat \beta}_{/2}}
\newcommand{\blonetwo}{\bm{\hat \beta}_{/1,2}}

\newcommand{\XJIJ}{\bm{X}_{/ijj'}}

\newcommand{\XJ}{\bm{X}_{/j}}
\newcommand{\yj}{\bm{y}_{/j}}
\newcommand{\Xji}{\bm{ X}_{\ \tilde i /j }}
\newcommand{\yji}{\bm{ y}_{\ \tilde i /j }}

\newcommand{\df}{\text{df}}
\newcommand{\poly}{\rm{poly}}
\newcommand{\snr}{\text{snr}}

\newcommand{\polyn}{\rm{poly}(\log n)}
\newcommand{\tXI}{{\bm{\bar X}}_{/i}}
\newcommand{\tXJIJ}{{\bm{\bar X}}_{/ijj'}}

\newcommand{\tGI}{{ \bm{\Gamma}}_{i/\bm{\tilde \delta_\ell}, \bm{\tilde \delta_r } }}
\newcommand{\tGII}{{ \bm{\Gamma}}_{ i/{\bm{\tilde \delta}_{1,\ell},\bm{\tilde \delta}_{1,r}  } }}
\newcommand{\tGIII}{{ \bm{\Gamma}}_{i/{\bm{\tilde \delta}_{2,\ell}, \bm{\tilde \delta}_{2,r} }}}
\newcommand{\zGI}{{\bm{ \Gamma}}_{\bm{\tilde \zeta} /i}}
\newcommand{\xGI}{{ \bm{\Gamma}}_{\bm{\tilde \xi} /i}}
\newcommand{\xxGI}{{ \bm{\Gamma}}_{ \bm{\xi} /i}}

\newcommand{\XI}{\bm{X}_{/i}}
\newcommand{\yi}{\bm{y}_{/i}}

\newcommand{\ld}{\dot{\ell}}
\newcommand{\ldd}{\ddot{\ell}}
\newcommand{\efd}{\dot{f}}
\newcommand{\efdd}{\ddot{f}}
\newcommand{\rd}{\dot{r}}
\newcommand{\psd}{\dot{\psi}}
\newcommand{\rdd}{\ddot{r}}
\newcommand{\rddd}{\dddot{r}}
\newcommand{\lddd}{\dddot{\ell}}
\newcommand{\dli}{\bm{\Delta}_{ \slash i}}
\newcommand{\pd}{\dot{\phi}}
\newcommand{\bs}{{\rm bias}}
\newcommand{\loi}{\widetilde{\text{LO}}_i}

\newcommand{\fd}{\dot{f}_e}
\newcommand{\fdd}{\ddot{f}_e}

\begin{abstract}
We study the problem of out-of-sample risk estimation in the high dimensional regime where both the sample size $n$ and number of features $p$ are large, and $n/p$ can be less than one. Extensive empirical evidence confirms the accuracy of leave-one-out cross validation (LO) for out-of-sample risk estimation. Yet, a unifying theoretical evaluation of the accuracy of LO in high-dimensional problems has remained an open problem. This paper aims to fill this gap for penalized regression in the generalized linear family. With minor assumptions about the data generating process, and without any sparsity assumptions on the regression coefficients, our theoretical analysis obtains finite sample upper bounds on the expected squared error of LO in estimating the out-of-sample error. Our bounds show that the error goes to zero as $n,p \rightarrow \infty$, even when the dimension $p$ of the feature vectors is comparable with or  greater than the sample size $n$. One technical advantage of the theory is that it can be used to clarify and connect some results from the recent literature on  scalable approximate LO.

%The paper considers the problem of out-of-sample risk estimation under the high dimensional settings where  standard techniques such as $K$-fold cross validation suffer from large biases. Motivated by the low bias of the leave-one-out cross validation (LO) method, we propose a computationally efficient closed-form approximate leave-one-out formula (ALO) for a large class of regularized estimators. Given the regularized estimate, calculating ALO requires minor computational overhead.  With minor assumptions about the data generating process, we obtain a finite-sample upper bound for $|\text{LO} - \text{ALO}|$. Our theoretical analysis illustrates that $|\text{LO} - \text{ALO}| \rightarrow 0$ with overwhelming probability, when $n,p \rightarrow \infty$, where the dimension $p$ of the feature vectors may be comparable with or even greater than the number of observations, $n$. Despite the high-dimensionality of the problem, our theoretical results do not require any sparsity assumption on the vector of regression coefficients. Our extensive numerical experiments  show that $|\text{LO} - \text{ALO}|$ decreases as $n,p$ increase, revealing the excellent finite sample performance of ALO. We further illustrate the usefulness of our proposed out-of-sample risk estimation method  by an example of real recordings from spatially sensitive neurons (grid cells) in the medial entorhinal cortex of a rat. 
\end{abstract}

\noindent%

%--------------------------------%--------------------------------%--------------------------------%--------------------------------
\section{Introduction}\label{section:intro}
Balancing the sensible level of model \textit{complexity} against model \textit{fitness} is a fundamental challenge faced by any learning algorithm. A model that is too simple can fail to  capture the essential pattern in  the data, and a model that is too complex is oversensitive to the idiosyncrasies of the particular data, resulting in highly variable patterns that are mere mirages in the noise. The learning algorithm's ability to perform well on \textit{new, previously unseen data} is typically used to set the model complexity. This performance is known as the \textit{out-of-sample error}. 

To be concrete, let $D = \{ (y_1, \bm{x_1}), \ldots, (y_n, \bm{x_n})\}$ be our dataset where $\bm{x_i} \in \R^p$ and $y_i \in \R$ denote the features and response, respectively. The goal is to obtain an estimate of the response for a newly observed feature vector. We assume observations are independent and identically distributed draws from some joint unknown distribution $q(y_i, \bm{x_i})$. We model this distribution as $q(y_i, \bm{x_i}) = q_1 (y_i | \bm{x}_i^\top \bm{\beta}_*) q_2(\bm{x}_i)$, and estimate $\bm{\beta}_*$ using the optimization problem
\begin{equation}\label{eq:ori_opt}
\bl \triangleq  \underset{\bm{\beta} \in \R^p}{\argmin}  \Bigl \{ \sum_{i=1}^n  \ell ( y_i \mid \bm{x_i}^\top \bm{\beta} ) + \lambda r(\bm{\beta})  \Bigr \},
\end{equation}
where $\ell$ is called the loss function, and $r(\bm{\beta})$ is called the regularizer. Both the regularizer $r(\bm{\beta})$ and the regularization parameter $\lambda$ have significant effects on the performance of the estimate by controlling the complexity of the model. Hence, for picking a good regularizer, $r$, or tuning the parameter $\lambda$ one would like to estimate the \textit{out-of-sample prediction error}, defined as
\begin{equation}
\extra \triangleq \E [ \phi ( y_o,\bm{x}_o^\top \bl ) \mid D ],
\end{equation}
where $(y_o,\bm{x}_o)$ is a \textit{new, previously unseen} sample from the unknown distribution $q(y , \bm{x})$ independent of $\mathcal{D}$, and $\phi$ is a function that measures the closeness of $y_o$ to $\bm{x}_o^\top \bl$. A standard choice for $\phi$ is $\ell(y \mid \bm{x}^\top \bm{\beta})$.

%is leave-one-out cross-validation (LO), a time-consuming and memory-demanding routine of leaving a datum out,  fitting the model on the rest, and testing it on the left out datum, repeatedly \cite{S74}. 

%A classic approach to estimate the algorithm's performance on unobserved data 
\begin{figure}
\begin{center}\vspace{-5.5cm}
\hspace{-0.9cm}
        \includegraphics[width=0.5\textwidth]{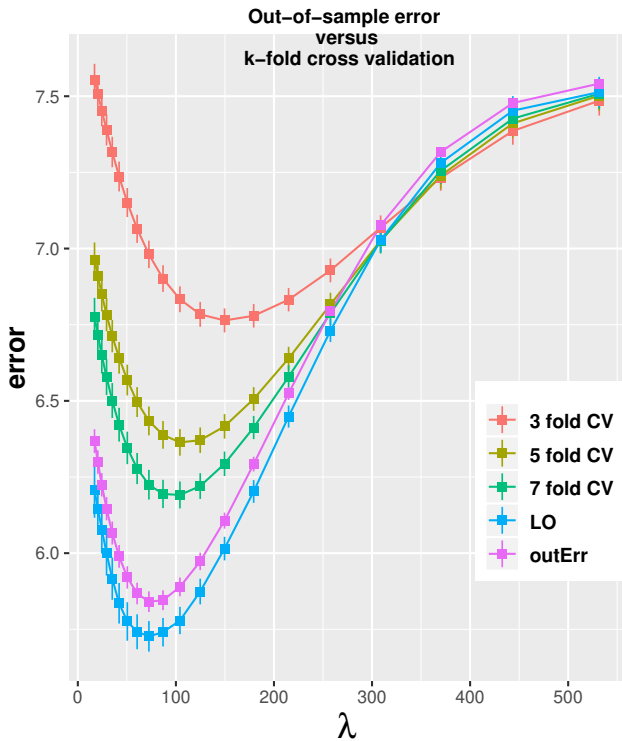}
          \caption{ Comparison of  $K$-fold cross validation (for $K=3,5,7$) and leave-one-out cross validation with the true (oracle-based)  out-of-sample  error for the elastic-net problem where $\ell ( y \mid \bm{x}^\top \bm{\beta} )=\frac{1}{2}(y-\bm{x}^\top \bm{\beta})^2$ and $r(\bm{\beta})=\| \bm{\beta} \|_1/2 + \| \bm{\beta} \|_2^2/4$. The upward bias of $K$-fold CV clearly decreases as number of folds increase. $y_i \sim \N(\bm{x_i}^\top \bm{\beta}^*,\sigma^2)$ and $\bm{x_i}\sim \N(\bm{0}, \bm{I})$. The number of nonzero elements of the true $\bm{\beta}^*$ is set to $k$ and their values is set to $\frac{1}{3\sqrt{2}}$. Dimensions are $(p,n,k)=\bigl(2000,500,100\bigr)$ and  $\sigma^2=2$. Extra-sample test data is $y_o \sim \N(\bm{x}_o^\top \bm{\beta}^*, \sigma^2)$ where  $\bm{x}_o \sim \N(\bm{0}, \bm{I})$. The true (oracle-based) out-of-sample prediction error is $\extra =  \E [  ( y_o-\bm{x}_o^\top \bl )^2 | D ]= \sigma^2 + \|\bl-\bm{\beta}^* \|_2^2$.  All depicted quantities are averages based on 100 random independent samples, and error bars depict one standard error. }
          \label{fig:loovscv5}
          \end{center}
\end{figure}
The problem of risk estimation has been extensively studied in the past fifty years and popular estimates, such as $k$-fold cross validation \cite{S74} are used extensively in practical systems. However, the emergence of high-dimensional estimation problems in which the number of features $p$ is comparable or even larger than the number of observations $n$, deemed many standard techniques in-accurate. For instance, Figure \ref{fig:loovscv5} compares the estimates obtained from $k$-fold cross validation for different values of $k$. As is clear in this figure, given the importance of each observation in high-dimensional settings, standard techniques, such as $5$-fold suffer from a large bias.

One of the existing estimates of $\extra$ that seems to be accurate in high-dimensional settings is the leave-one-out cross validation (LO), which is defined through the following formula:
\begin{equation}
\lo \triangleq \frac{1}{n} \sum_{i=1}^n \phi(y_i, \bm{x}_i^\top \bli),\label{eq:loo}
\end{equation}
where
\begin{eqnarray}\label{eq:lo_opt}
\bli \triangleq  \underset{\bm{\beta} \in \R^p}{\argmin}  \Bigl \{ \sum_{j\neq i}  \ell ( y_j \mid \bm{x}_j^\top \bm{\beta} ) + \lambda r(\bm{\beta})  \Bigr \} \label{eq:bli},
\end{eqnarray}
is the leave-$i$-out estimate. The simulation results reported in Figure \ref{fig:loovscv5} and elsewhere \cite{RM18, WZMLM18,SB19, BRST17,TK18} have demonstrated the good performance of LO in a wide range of high-dimensional problems. Despite the existence of extensive simulation results, the theoretical properties of LO have not been studied in the high-dimensional settings.

In this paper, we study the expected squared error of LO in estimating the out-of-sample error, in the high-dimensional setting, where both $n$ and $p$ are large, and $n/p$ can be less than one. We focus on regularized regression in the generalized linear family, and we make no sparsity assumption on the vector of regression coefficients. In short, we obtain an almost sharp upper bound on the error $|\lo - \extra|$. These bounds not only show that  $|\lo - \extra| \rightarrow 0$ as $n,p \rightarrow \infty$, but they also capture the rate of this convergence. This finally establishes what has been observed in empirical studies; $\lo$ obtains accurate out-of-sample risk estimates even in high-dimensional problems.  

An important advantage of our theoretical results is that they can be used to clarify and connect some results from the recent literature on computationally efficient approximation to LO. For instance, \cite{RM18} showed that in the same high dimensional regime, $|\alo - \lo| \rightarrow 0$ as $n,p \rightarrow \infty$, where $\alo$ stands for a computationally efficient approximation of LO we formally refer to in Section \ref{ssec:ALO1}. A major consequence of our theory is that it shows that $\alo$ is a consistent estimator of $\extra$. We make these statements more concrete in the next sections.

%The main objective of this paper is to explore the theoretical properties of the LO for high-dimensional estimation problems. We show that under mild regularity conditions, $|\lo-\extra| \rightarrow 0$, under the high-dimensional setting, where both $n$ and $p$ are large, but $n/p$ is a fixed number, possibly less than one. As it pertains to the performance of $\lo$ in estimating $\extra$ in the high dimensional setting \cite{XMRD19} showed that in the high dimensional setting, $|\lo - \extra| \rightarrow 0$ with high probability, for regularized linear estimation and assuming that $\bm{x}_i \sim \N(\bm{0}, \bm{I}/n)$, $y_i = \bm{x}_i^\top \bm{\beta}^* + w_i$ for a zero mean sub-Gaussian noise $w_i$ and  $\ell ( y \mid \bm{x}^\top \bm{\beta} ) = l( y-\bm{x}^\top \bm{\beta} )$ for some convex $l(.)$. But here we study non-linear models such as logistic regression, assume correlated predictors, and don't make any other assumption about the data generating process. The main theoretical contribution of this paper is that we obtain a finite-sample upper bound on the bias and mean square error of $\lo$ in estimating $\extra$.

%\subsection{Related work}

 \subsection{Notation}
 We first review the notations that will be used in the rest of the paper.    Let $\bm{x_i}^\top \in \R^{1 \times p}$ stand for the $i$th row of $\bm{X} \in \R^{n \times p}$. $\bm{y}_{/i} \in \R^{(n-1) \times 1}$ and $\XI \in \R^{(n-1) \times p}$ stand for $\bm{y}$ and $\bm{X}$, excluding the $i$th entry $y_i$ and the $i$th row $\bm{x_i}^\top$, respectively, and let $\bm{X}_{/ij}$ be defined likewise.
 Additionally, let $\bm{\hat \beta}_{/ij}$ stand for the regularized estimate in \eqref{eq:ori_opt} when $(y_i,\bm{x}_i)$ and $(y_j,\bm{x}_j)$ are excluded.
 Moreover, define 
\begin{align*}
\pd(y,z) &\triangleq \frac{\partial \phi (y,z)  }{\partial z}, \\
\ld(y_i \mid \bm{x_i}^\top\bm{\beta}) &\triangleq \frac{\partial \ell (y_i \mid z)  }{\partial z}\Big\rvert_{z =  \bm{x}_i^\top \bm{\beta}}, \\  
\ldd_i(\bm{\beta}) &\triangleq \frac{\partial ^2\ell (y_i \mid z)  }{\partial z^2}\Big\rvert_{z = \bm{x}_i^\top \bm{\beta}}, \\
%\bm{ \ld}_{/ i}(.) &\triangleq& [\ld_1(.),\cdots,\ld_{i-1}(.),\ld_{i+1}(.),\cdots,\ld_n(.)]^\top, \\
  \bm{\ldd}_{/ i}(\cdot) &\triangleq  [\ldd_1(\cdot),\cdots,\ldd_{i-1}(\cdot), \ldd_{i+1}(\cdot),\dotsc,\ldd_n(\cdot)]^\top.
\end{align*}
Likewise, define $\bm{\ldd}_{/ ij}(\bm{\beta})$. The notation $\poly\log n$ denotes  polynomial of $\log n$ with a finite degree. Let $\sigma_{\max}(\bm{A})$ and $\sigma_{\min}(\bm{A})$ stand for the largest and smallest eigenvalues of $\bm{A}$, respectively. We state $x_n = O_p(a_n)$ when the set of values $x_n/a_n$ is stochastically bounded.  
%$\bm{y}_{/i} \in \R^{(n-1) \times 1}$ and $\XI \in \R^{(n-1) \times p}$ stand for $\bm{y}$ and $\bm{X}$, excluding the $i$th entry $y_i$ and the $i$th row $\bm{x_i}^\top$, respectively.

\subsection{Computational complexity of LO and its approximation}\label{ssec:ALO1}
The high computational cost of repeatedly refitting models is a major hurdle in using $\lo$ in high dimensional settings. A typical approach to alleviate this problem analytically approximates the leave-$i$-out model based on the full-data model. A large body of work has addressed computationally efficient approximations to the leave-one-out cross validation error for ridge regularized  estimation problems (and its variants) \cite{A74,CW79,GHW79,OYR86,B90,CH92,OW00,CT08,MG13,VMTSW16, MMB15}.  Extensions to a wide array of regularizers, such as LASSO \cite{OK18,RM18,SB19} and nuclear norm \cite{WZMLM18} were recently studied and the validity of these approximations in estimating LO (and its variants) were theoretically studied in  \cite{OK16,BRST17,RM18,GSLJB19,SB19,XMRD19}. 

For example, a single Newton step around $\bl$ was used in \cite{WZMLM18,RM18} to approximate  $\bli$ by
\begin{equation*}
\begin{aligned}
\bm{\tilde{\beta}}_{ \slash i} &\triangleq \bl
+ \Big( \sum_{j\neq i} \bm{x_j} \bm{x_j}^\top \ldd ( y_j \mid \bm{x}_j^\top \bl ) + \lambda \bm{\nabla^2 r}(\bl) \Big)^{-1} \\
&\quad\cdot \bm{x_i} \ld ( y_i \mid \bm{x}_i^\top \bl )
\end{aligned}
\end{equation*}
and using the Woodburry lemma, the following scalable approximate LO (ALO) formula was obtained:
\begin{align}\label{eq:aloformulafinal}
    &\alo \triangleq \frac{1}{n} \sum_{i = 1}^n \phi\bigl(y_i, \bm{x}_i^\top \tilde{\bm{\beta}}_{\slash i}\bigr) \nonumber \\
    &= \frac{1}{n} \sum_{i = 1}^n \phi \biggl(y_i, \bm{x}_i^\top \hat{\bm{\beta}} + \bigl(\frac{H_{ii}}{1 - H_{ii}}\bigr) \frac{\ld(y_i \mid \bm{x}_i^\top \hat{\bm{\beta}})}{\ldd(y_i \mid \bm{x}_i^\top \hat{\bm{\beta}})} \biggr)
\end{align}
where
\begin{equation*}
\bm{H} \triangleq  \bm{X}  ( \bm{X}^\top \diag[\bm{\ldd}(\bl)] \bm{X} +  \lambda \bm{\nabla^2 r}(\bl)   )^{-1}  \bm{X}^\top \diag[\bm{\ldd}(\bm{\bl})]
\end{equation*}
This result was extended to nonsmooth regularizers. For example, \cite{WZMLM18,RM18, SB19} showed that for $r(\bm{\beta})=\|\bm{\beta}\|_1$, the same $\alo$ formula is a valid approximation of $\lo$ if the following $\bm{H}$ matrix is used:   
$$ \bm{H}=  \bm{X}_S \left ( \bm{X}_S^\top \diag[\bm{\ldd} (\bl)]   \bm{X}_S  \right)^{-1} \bm{X}_S^\top \diag[\bm{\ldd} (\bl)]$$ 
where $S$ is the active set of $\bl$ and $\bm{X}_S$ is the matrix $\bm{X}$ restricted to columns indexed by $S$. With minor assumptions about the data generating process and without any sparsity assumption on the vector of regression coefficients, \cite{RM18} (Theorem 3 and Corollary 1)  proved that for various regularizers and regression methods $|\alo - \lo| = O_p(\frac{\poly \log n}{n})$ in the high dimensional setting where $n/p = \delta$ is constant while $n,p \rightarrow \infty$.

Our finite sample bounds in the next section show that with similar (easy to check) regularity conditions,  for various regularizers and regression methods,  $|\lo - \extra| \rightarrow 0$ estimate go to zero as $n,p \rightarrow \infty$ but $ n/p = \delta$ is a fixed number.  As a byproduct of this result, we show that in this high dimensional regime $|\alo - \extra| \rightarrow 0$ as $n,p \rightarrow \infty$. We will more formally state these claims in Section \ref{sec:connectALOextra}. 

%In what follows, for various regularizers and regression methods, by explicitly quantifying constants $c_1(n)$ and $c_2(n)$, we discuss conditions \eqref{eq:c1def}, \eqref{eq:c2def}, and \eqref{eq:c3def} in Assumption \ref{ass:1}. We consider the ridge regularizer in Lemma \ref{lem:ridgereg} and the smoothed-$\ell_1$ (and elastic-net) regularizer in Lemma \ref{lem:smoothLASSO}. Concerning various regression methods, we consider logistic (Lemma \ref{lem:logistic1}), robust regression (Lemma \ref{lem:psudoHuber}), least-squares  (Lemmas \ref{lem:ridge1} and \ref{cor:ridge}), and Poisson (Lemmas \ref{lem:poisson1} and \ref{lem:poisson2}) regression. The results below show that under mild assumptions, for the cases mentioned above, $c_1(n)$ and $c_2(n)$ are polynomial functions of $\log n$, a result that plays a key role in our main theoretical result presented in Section \ref{sec:main}.

 % Under this high-dimensional setting a dilemma exists in using cross validation; If $K$ is a large number, then this requires the optimization problem \eqref{eq:ori_opt} to be solved $K$ times which is computationally demanding.

%This high dimensional setting has received a lot of attention \cite{karoui2013asymptotic,EBBLY13,bean2013optimal,donoho2016high,NR16,SBC17,DW18,RA19,TK18,OK18,XMRD19}.

\section{Main results}

\subsection{Our assumptions}\label{ssec:assumption}
Our goal is to evaluate the accuracy of LO in estimating the out-of-sample prediction error in the high-dimensional regime.
Our results are valid for finite values of $n$ and $p$.
Later, in order to make asymptotic conclusions, we suppose that $n/p = \delta$ is constant while $n,p \rightarrow \infty$.

%\begin{assumption} \label{A1}
%The function $L(\bm{\beta}) \triangleq \sum_{i=1}^n  \ell ( y_i|\bm{x_i}^\top \bm{\beta} ) + \lambda r(\bm{\beta})$ is strongly convex such that for some non-negative $\nu$, and  for all $\bm{ \beta}$ and $\bm{ \tilde \beta}$ in a convex hull of $\bl$ and $\bli$, we have
%\begin{eqnarray*}
%L(\bm{ \beta}) \geq L(\bm{ \tilde \beta}) + \partial L (\bm{\tilde \beta})^\top (\bm{\tilde \beta} - \bm{ \beta}) + \nu  \| \bm{\tilde \beta} - \bm{ \beta} \|_2^2
%\end{eqnarray*}
%\end{assumption}

%\begin{assumption} \label{A1}
%$\ld(y|\bm{x}^\top\bm{\beta})$ is a continuous function of $\bm{\beta}$.
%\end{assumption}

We now state our assumptions for theorem \ref{th:mse}.
For simplicity of exposition, we start by stating a \emph{strong} version of our assumptions, which often requires uniform bounds.
Weaker analogues are discussed in \ref{ssec:extension}.
As the assumptions may appear somewhat opaque and technical, we will discuss them in the context of usual assumptions and concrete examples of standard generalized linear models.

\begin{assumption} \label{A1}
 The vectors $\bm{x}_i$ are independent zero mean vectors with covariance $\bm{\Sigma} \in \R^{p \times p}$ such that $  \sigma_{\max} \left( \bm{\Sigma} \right) \leq \rho / p$ for a nonnegative constant $\rho$.
\end{assumption}

Assumption \ref{A1} characterizes the different distributions obtained for each $n$ and $p$.
The rows $\bm{x}_i^\top$ are scaled in a way that ensures $\E \| \bm{x}_i \|_2^2 = O(1)$  and $\var(\bm{x}_i^\top \bm{\beta}) = \bm{\beta}^\top \bm{\Sigma} \bm{\beta}= O(1)$, assuming that $\beta_i$ (for $i=1,\cdots,p$) is $O(1)$, e.g. $\|\bm{\beta} \|_2^2=O(p)$.
For instance, under the linear model $y_i = \bm{x}_i^\top \bm{\beta}+ \epsilon_i$, this scaling ensures that the signal-to-noise ratio in each observation remains fixed as $n,p$ grow (when the noise variance is a non-zero constant).
Unless we make explicit assumptions about the sparsity of $\bm{\beta}$, without the $1/p$ scaling, the Hessian of the optimization problem \eqref{eq:ori_opt} is dominated by the data, making the regularizer, and in turn $\lambda$, irrelevant.
In this paper, we make \textit{no sparsity assumption} on the vector of regression coefficients.
For similar finite signal-to-noise ratio scalings in the high-dimensional asymptotic analysis see \cite{K17,EBBLY13,bean2013optimal,donoho2016high,DMM11, bayati2012lasso,NR16,SBC17,DW18,RM18,XMRD19}.
Under this scaling, the optimal value of $\lambda$ will be $O_p(1)$ \cite{MMB15}.  

\begin{assumption} \label{A2} We assume the functions $\ell(y \mid z)$ and $\phi(y,z)$ are twice differentiable in $z$. We also assume that $\ell(y \mid z)$ and $r(\bm{\beta})$ are convex in $z$ and $\bm{\beta}$, respectively. Let $(y_o,\bm{x}_o)$ be a sample from  the unknown distribution $q(y,\bm{x})$ independent of $D=\{(y_1,\bm{x}_1),\cdots,(y_n,\bm{x}_n)\}$. We assume there exists constants $c_0$ and $c_1$, such that, for all $i,j$, uniformly:
\begin{align*}
%\bm{\beta_o} &\triangleq& \argmax_{\bm{\beta}}   n \E \ell(y, \bm{x}^\top \bm{\beta}) + r(\bm{\beta})
%\\
c_0 &\geq  \max \left ( \abs{\ld(y_i \mid \bm{x}_i^\top \bl)}, \abs{\ld(y_o \mid  \bm{x}_o^\top \bli)}\right), \nonumber \\
%c_0 &\geq  \abs{\ld(y_o \mid  \bm{x}_o^\top \bli)}, \\
c_1 &\geq\sup_{t \in [0,1]} \sqrt{ \E \bigl[ \pd(y_o, t\bm{x_o}^\top \bli + \bar{t}\bm{x_o}^\top \bl_{/ij})^2   \mid D_{/i}  \bigr]},  \nonumber \\
c_1 &\geq \sup_{t \in [0,1]} \sqrt{\E   \pd  (y_o,  t\bm{x_o}^\top \blone + \bar{t} \bm{x_o}^\top \blonetwo)^2   },  \\
c_1 &\geq  \sup_{t \in [0,1]}\sqrt{\E\bigl[ \pd(y_o, t\bm{x}_o^\top \bl + \bar{t}\bm{x}_o^\top \bli )^2 \mid D\bigr]},
\end{align*}
where $D_{/ i} \triangleq D \setminus \{(y_i,\bm{x}_i)\}$, and $\bar{t} = 1-t$. 
\end{assumption}

Assumption \ref{A2} characterizes the smoothness of the GLM problem (and its associated leave-one-out versions).
As we will show below there are many examples, such as logistic and robust regression, in which we can find $c_0$ and $c_1$.
However, in some other popular examples, such as linear or Poisson regression, $\abs{\ld(y_i \mid \bm{x}_i^\top \bl)}$ is a random quantity and we cannot find an absolute constant to dominate it everywhere.
As will be discussed later in Section \ref{ssec:extension}, we can weaken Assumption \ref{A2} at the expense of a slightly stronger moment condition on the feature vector $\bm{x}_i$. 

{\bf Example 1.} In the generalized linear model family, for the negative logistic regression log-likelihood $\ell ( y \mid \bm{x}^\top \bm{\beta} )=-y\bm{x}^\top \bm{ \beta} +\log (1 + e^{\bm{x}^\top \bm{ \beta}})$, where  $y \in \{0,1\}$, for $\phi(y,z)=\ell(y \mid z)$ it is easy to show that $\ld(y \mid z) \leq 2$ for any $y$ and $z$, leading to $c_0=c_1=2$. 

{\bf Example 2.} Our next example is about a smooth approximation of the Huber loss used in robust estimation, known as the pseudo-Huber loss:
\begin{equation}\label{eq:pseudo-huber-loss}
f_H(z) = \gamma^2 \left( \sqrt{1+ \frac{z^2}{\gamma^2} }-1\right),
\end{equation}
where $\gamma>0$ is a fixed number. If we use this loss for the linear regression problem, and set $\ell ( y \mid \bm{x}^\top \bm{\beta} )= \phi(y, \bm{x}^\top \bm{\beta} )=f_H(y-\bm{x}^\top \bm{ \beta})$. It is easy to show that $\ld(y \mid z) \leq \gamma$ for any $y$ and $z$, leading to $c_0=c_1=\gamma$. 

Our next example is concerned with another popular loss function in linear regression, namely the absolute deviation. However, since we would like our loss functions to be differentiable, we use the following smooth approximation of the absolute deviation loss, $\ell(y \mid z)= \abs{y-z}$, introduced in \cite{schmidt2007fast}:
\[
\ell_{\gamma}(y \mid z) \triangleq  \frac{1}{\gamma} \Big( \log(1+ e^{\gamma (y-z)}) + \log(1+ e^{-\gamma (y-z)}) \Big),
\]
where $\gamma>0$ is fixed.\footnote{ Note that  $\lim_{\gamma \rightarrow \infty} \sup_{y,z} \abs[\big]{\abs{y-z} - \ell_{\gamma}(y \mid z)} =0$.}

{\bf Example 3.}  For $\ell ( y \mid \bm{x}^\top \bm{\beta} )= \phi(y, \bm{x}^\top \bm{\beta} )=\ell_{\gamma}(y \mid z)$, we have $c_0=c_1=1$. In fact, it is straightforward to show that $\ld_{\gamma}(y \mid z) \leq 1$ for any $y$ and $z$. \\
%Note that in the canonical case of the generalized linear model family, if we set the cost function equal to the negative log-likelihood we have $\phi(y, z) = -yz+g(z)$ where $g(z)$ is a convex function, leading to $\pd(y, z)=-y+g'(z)$. For logistic regression, with $y \in \{ 0,1\}$, $g(z)=\log(1+e^z)$, $g'(z)=(1 + e^{-z})^{-1} \leq 1$ leading to $\pd(y, z) \leq 2$. If we assume here that $\phi(y,z)=\ell(y,z)$, then for logistic regression, $c_1=2$.

\begin{assumption} \label{A3} For $t \in [0,1]$ define the two matrices
\begin{align}
\bm{A}_{t,/ i} &\triangleq \XI^\top \diag[\bm{\ldd}_{/ i}( t \bli +(1-t)\bl)] \XI \nonumber \\
&\quad+ \lambda \bm{\nabla^2 r} (t \bli +(1-t)\bl), \nonumber \\
\bm{A}_{t,/ i,j} &\triangleq \bm{X}_{/ij}^\top \diag[\bm{\ldd}_{/ ij}( t \bl_{/ij} +(1-t)\bli)] \bm{X}_{/ij} \nonumber \\
&\quad+ \lambda \bm{\nabla^2 r}(t \bl_{/ij} +(1-t)\bli).
\end{align}
We assume that there exists a fixed number $\nu$, such that 
\begin{align*}
%\bm{\beta_o} &\triangleq& \argmax_{\bm{\beta}}   n \E \ell(y, \bm{x}^\top \bm{\beta}) + r(\bm{\beta})
%\\
%c_2 &\triangleq& \max \left( |\ldd(y, \bm{x}^\top \bm{\beta_o})| \right) \\
\nu &\leq \min \Big( \min_{1 \leq i \leq n}\inf_{t \in [0,1]}  \sigma_{\min} (\bm{A}_{t,/ i}),
\\
&\qquad\min_{1 \leq i \leq n}\inf_{t \in [0,1]}  \sigma_{\min} (\bm{A}_{t,/ i,j})
 \Big).
\end{align*}
\end{assumption}

Assumption \ref{A3} characterizes the curvature of the GLM problem (and its associated leave-one-out versions).
In some examples, such as the ones that have ridge or smoothed elastic-net as the regularizer, it is straightforward to confirm this assumption.
For instance, for the ridge regularization, $r(\bm{\beta}) = \| \bm{\beta} \|_2^2/2$, we have $\nu > \lambda$.
In Section \ref{ssec:extension}, we explain how this assumption can be relaxed (at the expense of requiring more stringent moment conditions on $\bm{x}_i$) to cover more examples.

Having stated our assumptions, we now move on to stating our main result before proposing a number of examples to demonstrate how this result can be applied in common GLM cases.

\subsection{Main theorem}

Based on these assumptions we can now evaluate the accuracy of LO in estimating $\extra$. The following theorem proves that the expected square error of $\lo$ in estimating $\extra$ is small even in high-dimensional asymptotic settings. 

%\begin{theorem}\label{th:bias}
%Let $\delta \triangleq n/p$. If Assumptions \ref{A1}, \ref{A2} and \ref{A3} hold, then
%\begin{equation*}
%\abs*{     \frac{1}{n} \E \sum_{i=1}^n \phi(y_i, \bm{x}_i^\top \bli)  - \E \phi(y_o, \bm{x}_o^\top \bl)      }^2  \leq   \biggl(\frac{c_0 c_1 \rho \delta^{1/2}}{\sqrt{n}\nu}\biggr)^2.
%\end{equation*}
%\end{theorem}
%------------------PROOF BIAS------------------------------------------------------------------------------------------------------------------------------

%------------------THEOREM MSE------------------------------------------------------------------------------------------------------------------------------

\begin{theorem} \label{th:mse}
Let $\delta \triangleq n/p$. If Assumptions  \ref{A1}, \ref{A2} and \ref{A3} hold, then
\begin{equation*}
\E \left (\frac{1}{n} \sum_{i=1}^n \phi(y_i, \bm{x}_i^\top \bli)  - \E \bigl[\phi(y_o, \bm{x}_o^\top \bl) \mid D\bigr]      \right )^2  \leq \frac{C_v}{n},
\end{equation*}
where the outer expectation is taken with respect to the data $D$ and:
\begin{equation*}
  \begin{split}
C_v &=  \E \var[ \phi(y_o, \bm{x}_o^\top \blone) \mid D_{/1}]  + 2 C_b \\
&\quad+
2 C_b^{1/2} \sqrt{\E \var[ \phi(y_o, \bm{x}_o^\top \blone) \mid D_{/1}]  +  C_b},
  \end{split}
\end{equation*}
and $C_b = \left(\frac{c_0 c_1 \rho \delta^{1/2}}{\nu}\right)^2$.
\end{theorem}
The proof can be found in Appendix \ref{sec:proof-main-theorem}.

The only term that is not explicitly computed in terms of the constants in our assumptions is $\E \var[ \phi(y_o, \bm{x}_o^\top \bli) \mid D_{/i}]$.
Hence, to obtain an explicit quantitative bound for a specific GLM problem requires computing this quantity.
We present two examples below.

\begin{corollary} \label{ex:logistic}(Ridge regularized logistic regression) Consider the negative logistic regression log-likelihood $\ell ( y|\bm{x}^\top \bm{\beta} )=-y\bm{x}^\top \bm{ \beta} +\log (1 + e^{\bm{x}^\top \bm{ \beta}})$, and the regularizer $r(\bm{\beta}) = \|\bm{\beta}\|_2^2/2$, where  $y \in \{0,1\}$. Furtherassume that $\bm{x}_i$ is iid $N(\bm{0}, \bm{\Sigma})$, where $\sigma_{\max} (\bm{\Sigma}) \leq \frac{\rho}{p}$. If $\phi(y,z)=\ell(y|z)$, then there exists a constant $C_v$ such that
\begin{equation*}
\E \left (\frac{1}{n} \sum_{i=1}^n \phi(y_i, \bm{x}_i^\top \bli)  - \E \bigl[\phi(y_o, \bm{x}_o^\top \bl) \mid D\bigr]      \right )^2  \leq \frac{C_v}{n}, 
\end{equation*}
where
\begin{eqnarray}
    C_v&=&6 + \frac{5 \rho \delta }{\lambda} + 2\left(\frac{4 \rho \delta^{1/2}}{\lambda}\right)^2 \nonumber
    \\
    &+& 2\left(\frac{4 \rho \delta^{1/2}}{\lambda}\right)\sqrt{6 + \frac{5 \rho \delta }{\lambda} + \left(\frac{4 \rho \delta^{1/2}}{\lambda}\right)^2}. \label{eq:log}
\end{eqnarray}
\end{corollary}
The proof of this corollary can be found in Section \ref{sec:proof:ex:logistic} of the supplementary material.

%From the first order optimality condition, we have $\bm{X}^\top \ld(\bl) + \lambda \bl=0$, leading to
%\begin{eqnarray*}
%\E \| \bli \|_2^2
%&\leq& \frac{4(n-1)}{\lambda} \E \sigma_{\max}(\XI \XI^\top) \nonumber \\ &\leq& \frac{4n}{\lambda}  \E \sigma_{\max}(\bm{X} \bm{X}^\top)
%\end{eqnarray*}
%because for the negative logistic regression log-likelihood we have $\frac{\partial_z \ell ( y|z )}{\partial z}=-y +\frac{1}{1 + e^{-z}}$. Let $\bm{\epsilon}_i \sim \N(\bm{0}, \bm{I}_p)$ where $i=1,\cdots,n$. Then,
%\begin{eqnarray*}
%\lefteqn{\E \sigma_{\max}(\bm{X} \bm{X}^\top)  = \E \sigma_{\max}(\bm{X}^\top \bm{X})} \nonumber \\
%&=& \E \sigma_{\max}(\bm{\Sigma}^{1/2} \bm{\Sigma}^{-1/2} \bm{X}^\top \bm{X} \bm{\Sigma}^{-1/2} \bm{\Sigma}^{1/2} )
%&=& \E \sigma_{\max}\left (\bm{\Sigma}^{1/2} (\sum_{i=1}^n \bm{\epsilon}_i \bm{\epsilon}_i^\top ) \bm{\Sigma}^{1/2} \right)
%&\leq&
%\frac{\rho}{p} \E \sigma_{\max} (\sum_{i=1}^n \bm{\epsilon}_i \bm{\epsilon}_i^\top )  = \rho (1 + \sqrt{\delta})^2,
%\end{eqnarray*}
%where the last step is due to Lemma %\ref{lem:expectedeigenvalue} in Appendix \ref{}. 
%Therefore, we can say that for ridge regularized logistic regression
%\begin{eqnarray*}
 %\var[ \phi(y_o, \bm{x}_o^\top \bli)| D_{/i}]  &\leq& 6 + \frac{7 \rho^2 }{\lambda} \delta(1 + \sqrt{\delta})^2
%\end{eqnarray*}

\begin{corollary} \label{example:psedoHuber} (Pseudo-Huber loss with strongly convex regularizer)
We consider again the pseudo-Huber loss defined in \eqref{eq:pseudo-huber-loss} with parameter $\gamma$. As this loss is typically used in regression settings, we consider a linear regression model $y_i = \xv_i^\top \bm{\beta^*} + \epsilon_i$,
where $\epsilon_i$ denotes i.i.d. zero-mean noise, and $\bm{x}_i \sim N(\bm{0}, \bm{\Sigma})$ with $\sigma_{\max} (\bm{\Sigma}) \leq \frac{\rho}{p}$. We additionally assume that the regularizer is strongly convex with parameter $\nu_r$,\footnote{Note that this is a fairly benign assumption in practice: it is common to introduce a slight ridge penalty which automatically
satisfies this assumption.} $\Var(\epsilon) = \sigma_\epsilon^2$, and $\frac{1}{p} \bm{\beta}^{*\top} \bm{\beta^*} \leq b$.
Under these conditions, there exists a fixed number 
$C_v$ (depending on $\gamma$, $\sigma_\epsilon$, $b$, $\rho$, $\delta$ and $\nu_r$) such that
\begin{equation*}
\E \left (\frac{1}{n} \sum_{i=1}^n \phi(y_i, \bm{x}_i^\top \bli)  - \E \bigl[\phi(y_o, \bm{x}_o^\top \bl) \mid D\bigr]      \right )^2  \leq \frac{C_v}{n}.
\end{equation*}

\end{corollary}
The proof of this corollary can be found in Section \ref{sec:proof:example:psedoHuber} of the supplementary material. 

To summarize, the examples presented in Corollary \ref{ex:logistic} and \ref{example:psedoHuber} satisfy the assumption needed for Theorem \ref{th:mse}.

\subsection{Extensions}\label{ssec:extension}

As we discussed in Section \ref{ssec:assumption}, we can weaken the assumptions without a major change in our proofs or the main conclusions of our result. In this section, we aim to present one such weaker set of assumptions that enables our analyses to cover several other popular examples, such as the Poisson and linear regression.  \\

\begin{assumptionbis}{A1}\label{Aprime1}
 We assume that $\bm{x}_i$ are i.i.d. zero mean vectors with covariance $\bm{\Sigma} \in \R^{p \times p}$ such that $  \sigma_{\max} \left( \bm{\Sigma} \right) \leq \rho / p$ for a non-negative constant $\rho$. Furthermore, there exists a fixed number $c_4$, such that $\mathbb{E} (\|\bm{x}_i\|_2^4) \leq c_4$. 
\end{assumptionbis}

Note that this assumption is more stringent than Assumption \ref{A1}. However, in essence the only extra requirement of this assumption is a bound on the fourth moments. Hence, it holds for a wide range of random features including sub-Gaussian and sub-exponential features. Thanks to this slightly stronger moment assumption we can weaken the other assumptions. \\

\begin{assumptionbis}{A2} \label{Aprime2} We assume the functions $\ell(y \mid z)$ and $\phi(y,z)$ are twice differentiable in $z$. Moreover, assume $\ell(y \mid z)$ and $r(\bm{\beta})$ are convex in $z$ and $\bm{\beta}$, respectively. Let $(y_o,\bm{x}_o)$ be a sample from  the unknown distribution $q(y,\bm{x})$ independent of $D=\{(y_1,\bm{x}_1),\cdots,(y_n,\bm{x}_n)\}$. We assume that there exist constants $\tilde{c}_0$ and $\tilde{c}_1$, such that for all $i,j$, uniformly 
 \begin{align*}
%\bm{\beta_o} &\triangleq& \argmax_{\bm{\beta}}   n \E \ell(y, \bm{x}^\top \bm{\beta}) + r(\bm{\beta})
%\\
\tilde{c}_0 &\geq   \mathbb{E} \abs{\ld(y_1 \mid \bm{x}_1^\top \bl)}^8, \nonumber  \\
\tilde{c}_0 &\geq \mathbb{E} \abs{\ld(y_o \mid  \bm{x}_o^\top \blone)}^8, \\
\tilde{c}_1 &\geq \sup_{t \in [0,1]} \sqrt{ \E \left [ \pd(y_o, t\bm{x_o}^\top \bli + \bar{t}\bm{x_o}^\top \bl_{/ij})^2   \mid D_{/i}  \right ]},  \nonumber \\
\tilde{c}_1 &\geq \sup_{t \in [0,1]} \sqrt{\E   \pd  (y_o,  t\bm{x_o}^\top \blone + \bar{t}\bm{x_o}^\top \blonetwo)^2   },  \\
\tilde{c}_1 &\geq  \sup_{t \in [0,1]}\sqrt{\E[ \pd(y_o, t\bm{x}_o^\top \bl + \bar{t} \bm{x}_o^\top \bli )^2 \mid D]},
\end{align*}
where $D_{/ i} \triangleq D \setminus \{(y_i,\bm{x}_i)\}$, and $\bar{t}= 1-t$.
\end{assumptionbis}

Compared to Assumption \ref{A2} that requires $\abs{\ld(y_i \mid \bm{x}_i^\top \bl)}$ to be bounded everywhere, this assumption requires the $8^{\rm th}$ moment of $\abs{\ld(y_i \mid \bm{x}_i^\top \bl)}$ to be bounded. This simple modification enables our theoretical results to be applied to a much broader set of regression techniques, including Poisson, linear, and negative binomial regression. These three examples will be studied later in this section.

\begin{assumptionbis}{A3} \label{Aprime3}
Let $\bm{A}_{t,/ 1}$ and $\bm{A}_{t,/ 1,2}$ be as defined in Assumption \ref{A3}.
We assume that there exists a fixed number $\tilde{\nu}$, such that 
\begin{align*}
%\bm{\beta_o} &\triangleq& \argmax_{\bm{\beta}}   n \E \ell(y, \bm{x}^\top \bm{\beta}) + r(\bm{\beta})
%\\
%c_2 &\triangleq& \max \left( |\ldd(y, \bm{x}^\top \bm{\beta_o})| \right) \\
\tilde{\nu} &\geq  \E \Big(\inf_{t \in [0,1]}  \sigma_{\min} \left( \bm{A}_{t,/ 1}  \Big)  \right)^{-8},
\\
\tilde{\nu} &\geq \E \Big(\inf_{t \in [0,1]}  \sigma_{\min} \left(\bm{A}_{t,/ 1,2} \right)
 \Big)^{-8}.
\end{align*}
\end{assumptionbis}
Again, compared to Assumption \ref{A3}, this assumption only bounds the moments of the minimum eigenvalue of the matrix. The following example shows an example in which it is impossible to find a positive lower bound for the minimum eigenvalue, but still the moments of the inverse of the minimum eigenvalue are bounded.

\begin{example}\label{ex:min_eig_oversamp}
Suppose that $\delta=n/p >1$ and that the loss function is strongly convex with parameter $c$, and the regularizer is convex. Finally, suppose that $\bm{x}_i \sim N(0, \bm{\Sigma})$, with $\sigma_{\min} (\bm{\Sigma}) = \frac{\rho}{p}$. Then, there exists a fixed number $\tilde{\nu}$ that satisfies Assumption \ref{Aprime3} for large enough values of $n$ and $p$.  
\end{example}

The proof can be found in Section \ref{sec:proof:ex:min_eig_oversamp} of the supplementary material.

As we discussed before one can prove the accuracy of $\lo$ under Assumptions \ref{Aprime1}, \ref{Aprime2}, and \ref{Aprime3}. The following theorem formalizes this claim.

\begin{theorem} \label{th:mse2}
Let $\delta \triangleq n/p$. If Assumptions  \ref{Aprime1}, \ref{Aprime2} and \ref{Aprime3}, then
\begin{equation*}
\E \left (\frac{1}{n} \sum_{i=1}^n \phi(y_i, \bm{x}_i^\top \bli)  - \E [\phi(y_o, \bm{x}_o^\top \bl) \mid D]      \right )^2  \leq \frac{\tilde{C}_v}{n},
\end{equation*}
where the outer expectation is taken with respect to the data $D$ and:
\begin{equation*}
  \begin{split}
\tilde{C}_v &=  \E \var[ \phi(y_o, \bm{x}_o^\top \blone) \mid D_{/1}]  + 2 \tilde{C}_b \\
&\quad+
2 \tilde{C}_b^{1/2} \sqrt{\E \var\bigl[ \phi(y_o, \bm{x}_o^\top \blone) \mid D_{/1}\bigr]  +  \tilde{C}_b}.
  \end{split}
\end{equation*}
and $\tilde{C}_b = c_1^2 \rho \delta_0 \tilde{c}_0 \tilde{v} c_4$.
\end{theorem}

The proof can be found in Section \ref{sec:GenericTheorem} of the supplementary material. As we described before, this theorem can cover several generalized linear models, that could not be covered by Theorem \ref{th:mse}. We mention three important examples below. 

\begin{corollary} (Square loss with elastic-net penalty)\label{ex:squareLoss}
Consider the data generating mechanism $y_i = \bm{x}_i \bm{\beta}^* + \epsilon_i$, where $\bm{x}_i^\top \sim N(\bm{0}, \bm{\Sigma})$, $\epsilon_i \overset{iid}{\sim} N(0, \sigma_\epsilon^2)$, and $\frac{1}{p} \|\bm{\beta}^*\|_2^2 \leq b$. Suppose that we use the smoothed elastic-net optimization 
\[
\min_{\bm{\beta}}  \sum_{j=1}^n \frac{(y_j - \bm{x}_j^\top \bm{\beta})^2}{{2}} + \lambda \sum_{j=1}^p r(\beta_i),
\]
where for $\gamma >0$, $r(\beta) = \gamma \beta^2 + (1-\gamma) r^\alpha (\beta)$, and $r^\alpha(\beta) =  \frac{1}{\alpha} \Big( \log(1+ e^{\alpha \beta}) + \log(1+ e^{-\alpha \beta}) \Big)$ is a smooth approximation of the $\ell_1$-norm. Then, there exists a fixed number, $\tilde{C}_v$, such that
\begin{equation*}
\E \left (\frac{1}{n} \sum_{i=1}^n \phi(y_i, \bm{x}_i^\top \bli)  - \E [\phi(y_o, \bm{x}_o^\top \bl) \mid D]      \right )^2  \leq \frac{\tilde{C}_v}{n}.
\end{equation*}
\end{corollary}
Since the proof of this claim is long, we defer it to Section \ref{ssec:proof:ExampleElasticNet} of the supplementary material.

\begin{corollary}\label{ex:PoissonLoss}[Poisson regression with soft-rectifying link]
  Consider the data-generating mechanism $y_i \sim \mathrm{Poisson}(f(\bm{x}_i^\top\bm\beta^*))$,
  where $f(z) = \log(1 + e^z)$ denotes the soft-rectifying link, $\xv_i \overset{iid}{\sim} N(\bm{0}, \bm{\Sigma})$, and $\frac{1}{p} {\bm{\beta}^*}^\top \bm{\beta}^* \leq b$. Finally, assume that $r$ denotes the smoothed elastic-net regularizer introduced in Corollary \ref{ex:squareLoss}. Under these assumptions, there exists a fixed number, $\tilde{C}_v$, such that:
\begin{equation*}
\E \left (\frac{1}{n} \sum_{i=1}^n \phi(y_i, \bm{x}_i^\top \bli)  - \E [\phi(y_o, \bm{x}_o^\top \bl) \mid D]      \right )^2  \leq \frac{\tilde{C}_v}{n}.
\end{equation*}
  \end{corollary}

The proof can be found in Section \ref{sec:proof:ex:PoissonLoss} of the supplementary file. 

\begin{remark}
  We have assumed here that $\bm{x}_i$ is multivariate Gaussian. As might be clear to the reader from the proof, this normality assumption on $\xv$ may be relaxed to an $8^{\rm th}$ moment assumption at the cost of a slightly more complicated proof.  
\end{remark}

\begin{corollary}[Negative-Binomial Regression]\label{ex:negbino}
    We consider the problem of negative binomial regression with fixed shape parameter $\alpha > 0$ and exponential link.
    Here, the negative log-likelihood is given by:
    \begin{equation*}
        \ell(y \mid z) = (y + \alpha^{-1}) \log(1 + \alpha e^z) - y z + C(\alpha, y),
    \end{equation*}
    where $C(\alpha, y)$ denotes a constant which only depends on $\alpha$ and $y$.
    Assume the data generating process is such that $\E [y^8] \leq \kappa$, and that $\phi(y,z)=\ell(y|z)$. Finally, similar to Corollary, \ref{ex:squareLoss} we use the smoothed elastic-net as the regularizer.
    Under these assumptions, there exists a fixed number, $\tilde{C}_v$, such that
\begin{equation*}
\E \left (\frac{1}{n} \sum_{i=1}^n \phi(y_i, \bm{x}_i^\top \bli)  - \E [\phi(y_o, \bm{x}_o^\top \bl) \mid D]      \right )^2  \leq \frac{\tilde{C}_v}{n}.
\end{equation*}
 \end{corollary}
The proof can be found in Section \ref{sec:proof:ex:negbino} of the supplementary material.

\section{Connection of $\alo$ and $\extra$}\label{sec:connectALOextra}
We mentioned in Section \ref{ssec:ALO1} that  different approximations of $\lo$ have been proposed in the literature to reduce the computational complexity of $\lo$.  Among such approximations, the $\alo$ formula introduced in \eqref{eq:aloformulafinal},  is analyzed in \cite{RM18} under a similar asymptotic framework as the one discussed in our paper:
\begin{theorem}\label{thm:aloVSlo}\cite{RM18}
Suppose that $n/p = \delta$ is constant while $n,p \rightarrow \infty$. Under the assumption $\bm{x}_i \sim N(0, \bm{\Sigma})$, for the regression problems discussed in Corollaries \ref{ex:logistic}, \ref{example:psedoHuber}, \ref{ex:squareLoss}, and \ref{ex:PoissonLoss} we have
\[
|\alo-\lo| = O_p\left(\frac{\poly \log n}{\sqrt{n}}\right). 
\]
\end{theorem}

Note that the ultimate goal of $\alo$ is to use it as an estimate of $\extra$. Hence, while Theorem \ref{thm:aloVSlo} confirms the accuracy of ALO in approximating $\lo$ it does not explain whether the estimates obtained by $\alo$ or $\lo$ can be trusted in high-dimensional settings. However, we can combine this result with Theorems \ref{th:mse} and \ref{th:mse2} to prove the accuracy of $\alo$ in estimating $\extra$. Toward this goal we first prove the following claim.

\begin{theorem}\label{thm:opLO}
Suppose that $n/p = \delta$ is constant while $n,p \rightarrow \infty$. Under the assumption $\bm{x}_i \sim N(0, \bm{\Sigma})$, for the regression problems discussed in Corollaries \ref{ex:logistic}, \ref{example:psedoHuber}, \ref{ex:squareLoss}, and \ref{ex:PoissonLoss} we have
\[
|\lo-\extra| = O_p\left(\frac{1}{\sqrt{n}}\right). 
\]
\end{theorem}
\begin{proof}
For a fixed number $M$
\begin{align}\label{eq:proofLO_extral}
\MoveEqLeft{\P \Bigl( \abs[\big]{\lo-\extra} > \frac{M}{\sqrt{n}}\Bigr) } \nonumber \\
&= \P \Bigl( \abs[\big]{\lo-\extra}^2 > \frac{M^2}{n}\Bigr) \nonumber \\
&\leq \frac{n}{M^2} \E |\lo-\extra|^2 \nonumber \\
&\leq \frac{n}{M^2}\frac{\min (C_\nu, \tilde{C}_\nu)}{n}= \frac{\min (C_\nu, \tilde{C}_\nu)}{M^2}.
\end{align}
The first inequality in the above equations is due to Markov inequality, and the second inequality is a result of Theorems \ref{th:mse} and \ref{th:mse2}. As we discussed in Corollaries \ref{ex:logistic}, \ref{example:psedoHuber}, \ref{ex:squareLoss}, and \ref{ex:PoissonLoss} either $C_\nu$ or $\tilde{C}_\nu$ are finite numbers. Hence, as $M$ increases, the final probability can be reduced to the desired level.  
\end{proof}

 Before we proceed to establish the accuracy of $\alo$ we have to clarify Theorem \ref{thm:opLO}. Note that even under the idealized (but incorrect) assumption that the individual estimates $\phi(y_i, \bm{x}_i^\top \bli)$ are independent and $\bli$s are the same as $\bl$, the central limit theorem indicates that $|\lo-\extra| \sim \frac{1}{\sqrt{n}}.$\footnote{The notation $|\lo-\extra| \sim \frac{1}{\sqrt{n}}$ means that we have both $|\lo-\extra| \sim O_p(\frac{1}{\sqrt{n}})$ and $\frac{1}{\sqrt{n}} = O_p(|\lo-\extra|)$.} Hence, we should not expect the error of $\lo$ to be $o_p (\frac{1}{\sqrt{n}})$. Therefore, the above theorem seems to offer the sharpest result that is possible for $\lo$. Note that the sharpness is with regard to the rate of convergence and not the constants.

Combining the results of Theorem \ref{th:mse2} and Theorem \ref{thm:opLO}  we can finally quantify the accuracy of $\alo$ in estimating $\extra$. 

\begin{corollary}\label{thm:aloVSextra}
Suppose that $n/p = \delta$ is constant while $n,p \rightarrow \infty$. Under the assumption $\bm{x}_i \sim N(0, \bm{\Sigma})$, for the regression problems discussed in Corollaries \ref{ex:logistic}, \ref{example:psedoHuber}, \ref{ex:squareLoss}, and \ref{ex:PoissonLoss} we have
\[
|\alo-\extra| = O_p\left(\frac{\poly \log n}{\sqrt{n}}\right). 
\]
\end{corollary}
The proof of this corollary is straightforward, and is hence skipped. 
Note that this corollary finally establishes the fact that $\alo$ obtains accurate estimates of $\extra$. While we have established this result for only four popular examples in this paper, Theorems \ref{th:mse}, \ref{th:mse2} and Theorem 3 of \cite{RM18} can be applied to a much broader class of regression problems. Hence, a similar result is expected for such scenarios as well. Finally, we should emphasize that by comparing Theorems \ref{thm:opLO} and \ref{thm:aloVSextra} one may notice that the accuracy of $\alo$ might be worse than $\lo$ by a logarithmic factor. At this stage, it is not clear whether this difference is an artifact of the proof of \cite{RM18} or it is a real extra error that has been introduced by the approximation of $\lo$. 

\section{Numerical Experiments}
In this section, we present two numerical experiments to show that the $O(\frac{1}{n})$ bound given in Theorem \ref{th:mse} and \ref{th:mse2} is sharp but not tight. Specifically,  we generate synthetic data, and compare $\extra$ and $\lo$ for elastic-net linear regression and  ridge logistic regression. In all the examples in this section, the rows of $\bm{X}$ are $\N(\bm{0}, \bm{\Sigma})$. Here we let $\bm{\Sigma} = \bm{I}/n$ and  $\phi(y,z)=\ell(y \mid z)$. The codes for the Figure 1 and Table 1,2 are available at  \texttt{https://github.com/RahnamaRad/LO}.

\paragraph{Square loss with elastic-net penalty.}
We set  $\ell ( y \mid \bm{x}^\top \bm{\beta} )=\frac{1}{2}(y-\bm{x}^\top \bm{\beta})^2$, $r(\bm{\beta})=\frac{ (1-\alpha)}{2} \| \bm{\beta} \|_2^2+ \alpha \|  \bm{\beta} \|_1$  and $\alpha=0.5$. The true unknown parameter vector $\bm{\beta}^* \in \R^p$ is sparse with $k=0.1n$ non-zero elements independently drawn from a zero mean unit variance Laplace distribution, leading to $\var(\bm{x}^\top \bm{\beta}^*)=0.1$ (regardless of the values of $n$ and $p$). To generate data, we sample  $\bm{y} \sim \N(\bm{X} \bm{\beta}^*, \sigma^2 \bm{I}  )$. Here the out-of-sample error is:
\begin{equation*}
\extra=\E \ell ( y_o \mid \bm{x}_o^\top \bm{\beta} ) =
\sigma^2 + \| \bm{\Sigma}^{1/2}(\bl - \bm{\beta}^*) \|_2^2.
\end{equation*}

As we increase $n$ and $p$, we keep the ratio $\delta=n/p=0.1$ constant. We numerically calculate MSE$\triangleq \E (\extra - \lo)^2$ as a function of $n$ (and $p=10n$) based on 100 synthetic data samples, for each $n$, $p$ and $\lambda=5$. We fitted a line to model $\log(\text{MSE}) \sim \log(n)$ and obtained a slope of -1.03 (SE$=0.04$) and intercept of -0.46 (SE$=0.54$) with an Adjusted R-squared of 0.95. The slope of -1.03 ($SE=0.04$) shows that the bound is sharp because it confirms the $1/n$ scaling of our theory. Table \ref{t:1} shows the numerical MSE as a function of $n$ and $p$.

%We estimate $MSE\triangleq \E (\extra - \lo)^2$ based on 100 Markov chain Monte Carlo (MCMC) samples for $\lambda=5$. We report the estimated samples $MSE$ (and corresponding standard errors) for different of values of $p$ in table \ref{t:1}. Our major Theorem predicts that $MSE = O(1/n) = O(1/p)$. The fitted regression line of $\log MSE$ versus $\log p$ (based on the numbers presented in Table \ref{t:1}) is $\log MSE = 1.9 (SE=0.81) -1.02(SE=0.11) \log p$ with $R^2=0.96$. 

\begin{table}
\begin{center}
    \begin{tabular}{rrr}
    \toprule
    $n$       &   $p$     &   \multicolumn{1}{c}{MSE (SE)} \\
    \midrule
    40         &    400    &    0.0156 (0.0021)            \\
    80         &    800    &    0.0064 (0.0008)            \\
    120        &    1200    &    0.0039 (0.0006)           \\
    160        &    1600    &    0.0038 (0.0006)           \\
    200        &    2000    &    0.0028 (0.0004)           \\
    \bottomrule
    \end{tabular}
\end{center}\caption{Square loss with elastic-net penalty: MSE$\triangleq  \E (\extra - \lo)^2$ (and standard errors).} \label{t:1}
\end{table}

\textbf{Logistic regression with ridge penalty.}
We set $\ell ( y \mid \bm{x}^\top \bm{\beta} )=-y\bm{x}^\top \bm{ \beta} +\log (1 + e^{\bm{x}^\top \bm{ \beta}})$ (the negative logistic log-likelihood) and $r(\bm{\beta})=\frac{1}{2}\|  \bm{\beta} \|_2^2$. To generate data, we sample    $y_i \sim Binomial \left(\frac{e^{ \bm{x_i}^\top \bm{\beta}^*}}{1 + e^{\bm{x_i}^\top \bm{\beta}^*}}\right)$. Here the out-of-sample error
\begin{eqnarray*}
\extra&=&\E \ell ( y_o|\bm{x}_o^\top \bl )  \\
&=&  -\frac{\bl^\top \bm{\beta}^* }{\| \bm{\beta}^* \|_2^2} \E \left[ \frac{Ze^{Z}}{1+e^Z} \right] + \E \log(1+e^W)
\end{eqnarray*}
where $Z\sim \N(0,\|\bm{\Sigma}^{1/2} \bm{\beta}^* \|_2^2)$ and $W\sim \N(0,\|\bm{\Sigma}^{1/2} \bl \|_2^2)$.

%-1.00065    0.04438
As we increase $n$ and $p$, we keep the ratio $n/p=1$ constant. We numerically calculate MSE$\triangleq \E (\extra - \lo)^2$ as a function of $n$ (and $p=n$) based on 100 synthetic data samples, for each $n$, $p$ and $\lambda=0.1$. We fitted a line to model $\log(\text{MSE}) \sim \log(n)$ and obtained a slope of -1.00  (SE$= 0.04$) and intercept of 0.34 (SE$=0.27$) with an Adjusted R-squared of 0.99. The slope of -1.00 shows that the bound is sharp because it confirms the $1/n$ scaling of our theory.  Table \ref{t:2} compares the numerical MSE and the theoretical bound from Theorem \ref{th:mse} and Corollary \ref{ex:logistic}. The theoretical upper bound was computed using \ref{eq:log} in Corollary \ref{ex:logistic} where in this example, we have $\lambda=0.1$, $\rho=1$, and $\delta=1$, leading to $C_v=6311.52$. The significant difference between the bound and the MSE shows that the bound is not tight.

\begin{table}
\begin{center}
    \begin{tabular}{rrrr}
    \toprule
    $n $       &   $p$     &   \multicolumn{1}{c}{MSE (SE)}     &    \multicolumn{1}{c}{Bound}        \\
    \midrule
    100        &    100    &    0.0136 (0.0019)  &  63.12     \\ %0.013591477
    300        &    300    &    0.0037 (0.0005)  &   21.04    \\ %0.003702342
    500        &    500    &    0.0026 (0.0005)   &  12.62    \\ %0.002652845
    700        &    700    &    0.0017 (0.0002)   &  9.02     \\ %0.001725133
    900        &    900    &    0.0015 (0.0002)    &  7.01    \\ %0.001490155
    1100       &    1100   &    0.0012 (0.0002)   &  5.74     \\ %0.001187843
    \bottomrule
    \end{tabular}
\end{center}\caption{Logistic regression with ridge penalty: MSE$\triangleq \E (\extra - \lo)^2$ (and standard errors) and the  upper bound based on \ref{eq:log} in Corollary \ref{ex:logistic} of Theorem \ref{th:mse}. } \label{t:2}
\end{table}

%\textbf{Poisson regression with elastic-net penalty.}
%We set $\ell ( y|\bm{x}^\top \bm{\beta} )=e^{y\bm{x}^\top \bm{ \beta}} - y \bm{x}^\top \bm{ \beta}$ (the negative Poisson log-likelihood), $r(\bm{\beta})=\frac{ (1-\alpha)}{2} \| \bm{\beta} \|_2^2+ \alpha \|  \bm{\beta} \|_1$ and $\alpha=0.5$. To generate data, we sample $y_i \sim Poisson \left(e^{ \bm{x_i}^\top \bm{\beta^*} }\right)$. 

\section{Conclusion}
Leave-one-out estimators (and their approximate versions) have seen renewed interest recently in the context of big data and high-dimensional problems.
We show that, in general, leave-one-out risk estimators have desirable statistical behaviours in the high-dimensional setting.
Although the leave-out-risk estimator itself is generally computationally intractable, this result also implies consistency for a (growing) number of approximate leave-one-out estimators, and demonstrate that such estimators offer a potentially good direction for building risk estimators for high-dimensional problems.

\subsection*{Acknowledgement}
K.R. was supported by the NSF DMS grant 1810888, and Eugene M. Lang Junior Faculty Research Fellowship. A.M. was supported by the NSF DMS grant 1810888.

\bibliographystyle{apalike}
\bibliography{myrefs}

\newpage

\onecolumn
\appendix
\begin{center}
\textbf{\LARGE Supplementary Material}    
\end{center}
\section{Notation}
Let $\ld_i(\bm{\beta}) \triangleq \ld(y_i | \bm{x_i}^\top \bm{\beta})$ and $\ld(\bm{\beta}) \triangleq [\ld(y_1 | \bm{x_1}^\top \bm{\beta}), \cdots, \ld(y_n | \bm{x_n}^\top \bm{\beta})]^\top$.

\section{Background material on Gaussian random variables, vectors and matrices}
In this section, we review a few important results regarding the functions of Gaussian matrices and Gaussian vectors that are used in our examples. The first result is about the moments of the inverse of the minimum eigenvalue of a Wishart matrix. 
\begin{lemma}(Lemma 19 of \cite{XMRD19})\label{lem:socomp}
Let $X_{ij} \overset{i.i.d.}{\sim} N(0, \frac{1}{n})$, and suppose that $n,p \rightarrow \infty$ while $n/p = \delta_0$ for $\delta_0>1$. Then, 
for a fixed $r\geq 0$, we have 
\begin{equation}\label{eq:socomp_eq1}
	\E\left[ \frac{1}{\sigma_{\min}^r(\mathbf{X}^\top \mathbf{X})} \right] = O(1).	
\end{equation}
\end{lemma}

Our next two lemmas are concerned with the moments of a Gaussian and $\chi^2$ random variables:

\begin{lemma}\label{lem:momenet:Gauss}
Let $Z \sim N(0,\sigma^2)$. Then, we have
\begin{eqnarray}
\E (|Z|^p) \leq \sigma^p (p - 1)!!,
\end{eqnarray}
where the notation $p !!$ denotes the double factorial. Furthermore, when $p$ is even the above inequality is in fact an equality. 
\end{lemma}
The proof of this claim is straightforward and can be found in many standard statistics text books. 

\begin{lemma}\label{lem:momenet:chi}
Let $Z \sim \chi^2_k$, i.e., it has a $\chi^2$ distribution with $k$ degrees of freedom. Then, for any integer $m \geq 1$ we have
\[
\E(Z^m) = k(k+2) (k+4) \ldots (k+2m-2).
\]
\end{lemma}

\section{Proof of Theorem \ref{th:mse}}
\label{sec:proof-main-theorem}
%------------------PROOF MSE------------------------------------------------------------------------------------------------------------------------------

Define
\begin{align*}
%D &\triangleq& \{ (y_1,\bm{x}_1),\cdots, (y_n,\bm{x}_n) \} \\
%D_{/i} &\triangleq& \{ (y_1,\bm{x}_1),\cdots, (y_n,\bm{x}_n) \} - \{ (y_i,\bm{x}_i)\} \\
V_1 &\triangleq  \frac{1}{n} \sum_{i=1}^n \phi(y_i, \bm{x}_i^\top \bli)  - \frac{1}{n} \sum_{i=1}^n  \E[ \phi(y_i, \bm{x}_i^\top \bli) \mid D_{/i}],      \\
V_2 &\triangleq  \frac{1}{n} \sum_{i=1}^n  \E[ \phi(y_i, \bm{x}_i^\top \bli)| D_{/i}]    -   \E [\phi(y_o, \bm{x}_o^\top \bl) \mid D].
\end{align*}
Then,
\begin{align*}
\lefteqn{\E \left (    \frac{1}{n} \sum_{i=1}^n \phi(y_i, \bm{x}_i^\top \bli)  - \E [\phi(y_o, \bm{x}_o^\top \bl) \mid D]      \right )^2} \nonumber \\
&\quad\leq \E \left ( V_1 + V_2      \right )^2 \leq \E V_1^2 + \E V_2^2 + 2 \sqrt{\E V_1^2  \E V_2^2 }.
\end{align*}
The proof concludes upon noting that Lemma \ref{lem:v1} and  \ref{lem:v2} yield
\begin{align*}
 \E V_1^2  &\leq \frac{1}{n} \biggl(   \E \var[ \phi(y_o, \bm{x}_o^\top \blone) \mid D_{/1}]  +
\Bigl(\frac{c_0 c_1 \rho \delta^{1/2}}{\nu}\Bigr)^2 \biggr),
\\
 \E V_2^2  &\leq \frac{1}{n} \left(\frac{c_0 c_1 \rho \delta^{1/2}}{\nu}\right)^2.
\end{align*}
%------------------LEMMA V1------------------------------------------------------------------------------------------------------------------------------
\begin{lemma}\label{lem:v1}
Under the assumptions of Theorem \ref{th:mse} we have that:
\begin{equation*}
\E \biggl(\frac{1}{n} \sum_{i=1}^n \phi(y_i, \bm{x}_i^\top \bli)  - \frac{1}{n} \sum_{i=1}^n  \E[ \phi(y_i, \bm{x}_i^\top \bli) \mid D_{/i}] \biggr)^2
\leq
\frac{1}{n}\biggl(   \E \var[ \phi(y_o, \bm{x}_o^\top \blone)| D_{/1}]  +
\Bigl(\frac{c_0 c_1 \rho \delta^{1/2}}{\nu}\Bigr)^2 \biggr).
\end{equation*}
\end{lemma}
% PROOF LEMMA V1
\begin{proof}[Proof of Lemma \ref{lem:v1}]
\begin{gather*}
\E \left(\frac{1}{n} \sum_{i=1}^n \phi(y_i, \bm{x}_i^\top \bli)  - \frac{1}{n} \sum_{i=1}^n  \E[ \phi(y_o, \bm{x}_o^\top \bli) \mid D_{/i}] \right)^2 
 = \frac{1}{n}  \E \left( \phi(y_1, \bm{x}_1^\top \blone)  -  \E[ \phi(y_1, \bm{x}_1^\top \blone) \mid D_{/1}] \right)^2
\\
+
\frac{n-1}{n}  \E  \left[\bigl( \phi(y_1, \bm{x}_1^\top \blone)  -  \E[ \phi(y_1, \bm{x}_1^\top \blone) \mid D_{/1}] \bigr) \bigl( \phi(y_2, \bm{x}_2^\top \bltwo)  -  \E[ \phi(y_2, \bm{x}_2^\top \bltwo) \mid D_{/2}] \bigr)\right]
\end{gather*}
Note that
\begin{equation*}
\E \Bigl( \phi(y_1, \bm{x}_1^\top \blone)  -  \E[ \phi(y_1, \bm{x}_1^\top \blone) \mid D_{/1}] \Bigl)^2 = \E \var[ \phi(y_o, \bm{x}_o^\top \blone) \mid D_{/1}].
\end{equation*}
Next we study
\begin{equation*}
 \E  \Bigl[( \phi(y_1, \bm{x}_1^\top \blone)  -  \E[ \phi(y_1, \bm{x}_1^\top \blone) \mid D_{/1}] ) ( \phi(y_2, \bm{x}_2^\top \bltwo)  -  \E[ \phi(y_2, \bm{x}_2^\top \bltwo) \mid D_{/2}] )\Bigl].
\end{equation*}
Recall $ \blonetwo \triangleq  \underset{\bm{\beta} \in \R^p}{\argmin}  \Bigl \{   \sum_{k\geq 3}  \ell ( y_k \mid \bm{x}_k^\top \bm{\beta} ) + \lambda r(\bm{\beta})  \Bigr \}$. For some $t\in[0,1]$, the mean-value theorem  yields
 \begin{align*}
 \phi(y_1,\bm{x}_1^\top \blone) &=  \phi(y_1,\bm{x}_1^\top \blonetwo) + \pd(y_1,t\bm{x}_1^\top \blone + (1-t)\bm{x}_1^\top \blonetwo) \bm{x}_1^\top (\blone-\blonetwo)
\\
 \E [\phi(y_1,\bm{x}_1^\top \blone) \mid D_{/1}] &= \E [ \phi(y_1,\bm{x}_1^\top \blonetwo) \mid D_{/1}]  + \E [\pd(y_1,t\bm{x}_1^\top \blone + (1-t)\bm{x}_1^\top \blonetwo) \bm{x}_1^\top (\blone-\blonetwo) \mid D_{/1}]
 \\
 &= \E [ \phi(y_o,\bm{x}_o^\top \blonetwo) \mid D_{/1,2}]  + \E [\pd \bigl(y_o,t\bm{x}_o^\top \blone + (1-t)\bm{x}_o^\top \blonetwo) \bm{x}_o^\top (\blone-\blonetwo) \mid D_{/1}],
 \end{align*}
 where $(y_o, \bm{x}_o)$ is independent of $D$, leading to
 \begin{align*}
  \phi(y_1,\bm{x}_1^\top \blone) - \E [\phi(y_1,\bm{x}_1^\top \blone) \mid D_{/1}]
  &= \phi(y_1,\bm{x}_1^\top \blonetwo)  - \E [ \phi(y_o,\bm{x}_o^\top \blonetwo) \mid D_{/1,2}]
  \\
  &\quad +  \pd(y_1,t\bm{x}_1^\top \blone + (1-t)\bm{x}_1^\top \blonetwo) \bm{x}_1^\top (\blone-\blonetwo)
  \\
  &\quad - \E [\pd(y_o,t\bm{x}_o^\top \blone + (1-t)\bm{x}_o^\top \blonetwo) \bm{x}_o^\top (\blone-\blonetwo)\mid D_{/1}].
 \end{align*}
 Define the quantities $A_0, B_0, C_0$ as:
 \begin{align*}
 A_0 &\triangleq
 \E \left [( \phi(y_1,\bm{x}_1^\top \blonetwo)  - \E [ \phi(y_o,\bm{x}_o^\top \blonetwo) \mid D_{/1,2}] )(\phi(y_2,\bm{x}_2^\top \blonetwo)  - \E [ \phi(y_o,\bm{x}_o^\top \blonetwo) \mid D_{/1,2}] ) \right]
 \\
  &=
 \E \left [ \E \left[ \bigl( \phi(y_o,\bm{x}_o^\top \blonetwo)  - \E [ \phi(y_o,\bm{x}_o^\top \blonetwo) \mid D_{/1,2}] \bigr) \bigl(\phi(\tilde y_o,\bm{\tilde x}_o^\top \blonetwo)  - \E [ \phi(\tilde y_o,\bm{\tilde x}_o^\top \blonetwo )\mid D_{/1,2}] \bigr)   \Bigm | D_{/1,2} \right ] \right]
 \\
 &=0.
 \end{align*}
 Likewise,
 \begin{align*}
  B_0 &\triangleq
 \E \Big \{ \left ( \phi(y_1,\bm{x}_1^\top \blonetwo)  - \E\bigl[ \phi(y_o,\bm{x}_o^\top \blonetwo) \mid D_{/12} \bigr] \right ) \\
 &\quad \times
  \Big ( \pd(y_2,t\bm{x}_2^\top \bltwo + (1-t)\bm{x}_2^\top \blonetwo) \bm{x}_2^\top (\bltwo-\blonetwo)
  \\
 &\quad - \E \left [  \pd(y_2,t\bm{x}_2^\top \bltwo + (1-t)\bm{x}_2^\top \bl_{/12}) \bm{x}_2^\top (\bltwo-\blonetwo) \mid D_{/2}  \right]\Big) \Big\}
 \\
 &=
 \E \Big [ \left ( \phi(y_1,\bm{x}_1^\top \blonetwo)  - \E [ \phi(y_o,\bm{x}_o^\top \blonetwo) \mid D_{/12}] \right ) \\
 &\quad\times
  \Big ( \pd(\tilde y_o, t\bm{\tilde x}_o^\top \bltwo + (1-t)\bm{\tilde x}_o^\top \blonetwo) \bm{\tilde x}_o^\top (\bltwo-\blonetwo)   \\
 &\quad - \E \left [  \pd(\tilde y_o,t\bm{\tilde x}_o^\top \blj + (1-t)\bm{\tilde x}_o^\top \blonetwo) \bm{\tilde x}_o^\top (\bltwo-\blonetwo) \mid D_{/2}  \right]\Big) \Big]
 \\
 &=
 \E \Big [ \E\Big [ \left ( \phi(y_1,\bm{x}_1^\top \blonetwo)  - \E [ \phi(y_o,\bm{x}_o^\top \blonetwo) \mid D_{/12}] \right ) \\
 &\quad \times
  \Big ( \pd(\tilde y_o, t\bm{\tilde x}_o^\top \bltwo + (1-t)\bm{\tilde x}_o^\top \blonetwo) \bm{\tilde x}_o^\top (\blone-\blonetwo)   \\
  &\quad - \E \bigl[  \pd(\tilde y_o,t\bm{\tilde x}_o^\top \bltwo + (1-t)\bm{\tilde x}_o^\top \blonetwo) \bm{\tilde x}_o^\top (\bltwo-\blonetwo) \mid D_{/2}  \bigr]\Big)  \Bigm| D_{/2} \Big] \Big]
 \\
 &=
 \E \Bigg \{  \E \left[ \left ( \phi(y_1,\bm{x}_1^\top \blonetwo)  - \E [ \phi(y_o,\bm{x}_o^\top \blonetwo) \mid D_{/12}] \right ) \Bigm | D_{/2}\right] \\
 &\quad \times
 \E\Big [ \Big ( \pd(\tilde y_o, t\bm{\tilde x}_o^\top \bltwo + (1-t)\bm{\tilde x}_o^\top \blonetwo) \bm{\tilde x}_o^\top (\blone-\blonetwo)   \\
 &\quad - \E \left [  \pd(\tilde y_o,t\bm{\tilde x}_o^\top \bltwo+ (1-t)\bm{\tilde x}_o^\top \blonetwo) \bm{\tilde x}_o^\top (\bltwo-\blonetwo) \mid D_{/2}  \right]\Big)  \Bigm | D_{/2} \Big]\Bigg\}
 \\
 &=0.
  \end{align*}
  Likewise,
   \begin{align*}
  C_0
  &\triangleq \E \Big [ \left ( \phi(y_2,\bm{x}_2^\top \blonetwo)  - \E [ \phi(y_o,\bm{x}_o^\top \blonetwo) \mid D_{/12}] \right ) \\
  &\quad \times
  \Big ( \pd(y_1,t\bm{x}_1^\top \blone + (1-t)\bm{x}_1^\top \blonetwo) \bm{x}_1^\top (\blone-\blonetwo) \\
  &\quad - \E \left [  \pd(y_1,t\bm{x}_1^\top \blone + (1-t)\bm{x}_1^\top \bl_{/12}) \bm{x}_1^\top (\blone-\blonetwo) \mid D_{/1}  \right]\Big) \Big] \\
  &= 0.
   \end{align*}
To conclude, note that:
 \begin{align*}
  &\E \biggl[\Bigl( \phi(y_1, \bm{x}_1^\top \blone)  -  \E[ \phi(y_1, \bm{x}_1^\top \blone) \mid D_{/1}] \Bigr) \Bigl( \phi(y_2, \bm{x}_2^\top \bltwo)  -  \E[ \phi(y_2, \bm{x}_2^\top \bltwo) \mid D_{/2}] \Bigr)\biggr] =A_0 + B_0 + C_0 \\
  &\quad +
    \E \Biggl\{
      \biggl( \pd(y_1,t\bm{x}_1^\top \blone + (1-t)\bm{x}_1^\top \blonetwo) \bm{x}_1^\top (\blone-\blonetwo) \\
  &\quad - \E \bigl[  \pd(y_1,t\bm{x}_1^\top \blone + (1-t)\bm{x}_1^\top \blonetwo) \bm{x}_1^\top (\blone-\blonetwo) \mid D_{/1}  \bigr]\biggr)\\
  &\quad \times \Biggl( \pd(y_2,t\bm{x}_2^\top \bltwo + (1-t)\bm{x}_2^\top \blonetwo) \bm{x}_2^\top (\bltwo-\blonetwo) \\
  &\quad - \E \bigl [  \pd(y_2,t\bm{x}_2^\top \bltwo + (1-t)\bm{x}_2^\top \blonetwo) \bm{x}_2^\top (\bltwo-\blonetwo) \mid D_{/2}  \bigr]\Biggr)
    \Biggr\}
     \end{align*}
  \begin{align*}
  &\leq \E \var \left [ \pd(y_1,t\bm{x}_1^\top \blone + (1-t)\bm{x}_1^\top \blonetwo) \bm{x}_1^\top (\blone-\blonetwo)  \mid D_{/1}  \right ]  \\
  &\leq \E \Bigl(  \E \bigl[ \pd(y_1, t\bm{x}_1^\top \blone + (1-t)\bm{x}_1^\top \blonetwo) \bm{x}_1^\top (\blone-\blonetwo)  \mid D_{/1}  \bigr]^2 \Bigr)  \\
  &\leq \E \Bigl(  \E  \bigl[ \pd(y_1, t\bm{x}_1^\top \blone + (1-t)\bm{x}_1^\top \blonetwo)^2 \mid D_{/1}  \bigr]
    \E      \bigl[ \bigl( \bm{x}_1^\top (\blone-\blonetwo) \bigr)^2  \mid D_{/1}  \bigr]
    \Bigr)  \\
  &\leq c_1^2 \E \Bigl( \E\bigl[ \bigl(\bm{x}_1^\top (\blone-\blonetwo)\bigr)^2  \mid D_1 \bigr] \Bigr)
\\
    &= c_1^2 \E \Bigl( \bigl(\bm{x}_1^\top (\blone-\blonetwo) \bigr)^2  \Bigr)  \\
    &= c_1^2 \E \left( (\blone-\blonetwo)^\top \bm{\Sigma} (\blone-\blonetwo)  \right)  \\
    &\leq c_1^2 \frac{\rho}{p} \E \norm{ \blone-\blonetwo }_2^2 \\
&\leq c_1^2 c_0^2 \frac{\rho}{p \nu^2}  \E\norm{\bm{x}_1}_2^2,
\end{align*}
where the last inequality is due to Lemma \ref{lem:perturb_i}. Using the fact that $\delta=n/p$, $\bm{x}_i \sim N(\bm{0}, \bm{\Sigma})$ and $\sigma_{\max}(\bm{\Sigma}) = \rho/p$, we get
\begin{equation*}
  \E  \left[( \phi(y_i, \bm{x}_i^\top \bli)  -  \E[ \phi(y_i, \bm{x}_i^\top \bli)| D_{/i}] ) ( \phi(y_j, \bm{x}_j^\top \blj)  -  \E[ \phi(y_j, \bm{x}_j^\top \blj)| D_{/j}] )\right] \leq \frac{1}{n} \Bigl(\frac{c_0 c_1 \rho \delta^{1/2}}{\nu}\Bigr)^2.
\end{equation*}
\end{proof}

%------------------LEMMA V2------------------------------------------------------------------------------------------------------------------------------
\begin{lemma}\label{lem:v2}
Under the assumptions of Theorem \ref{th:mse}, we have:
\begin{equation*}
\E \left( \frac{1}{n} \sum_{i=1}^n  \E[ \phi(y_i, \bm{x}_i^\top \bli) \mid D_{/i}]    -   \E [\phi(y_o, \bm{x}_o^\top \bl) \mid D]  \right)^2 \leq \frac{1}{n}\left(\frac{c_0 c_1 \rho \delta^{1/2}}{\nu}\right)^2.
\end{equation*}
\end{lemma}
% PROOF LEMMA V2
\begin{proof}[Proof of Lemma \ref{lem:v2}]
Note that we have for all $i$:
\begin{align*}
\E[ \phi(y_i, \bm{x}_i^\top \bli) \mid D_{/i}] &= \E[ \phi(y_o, \bm{x}_o^\top \bli) \mid D_{/i}] = \E[ \phi(y_o, \bm{x}_o^\top \bli) \mid D].
\end{align*}
Therefore, using the mean-value Theorem, for some $t \in [0, 1]$, we get
\begin{align*}
& \E \left( \frac{1}{n} \sum_{i=1}^n  \E[ \phi(y_i, \bm{x}_i^\top \bli) \mid D_{/i}]   - \E [\phi(y_o, \bm{x}_o^\top \bl) \mid D]  \right)^2  \\
&= \E \left( \frac{1}{n} \sum_{i=1}^n  \E[ \phi(y_o, \bm{x}_o^\top \bli) \mid D]   - \E [\phi(y_o, \bm{x}_o^\top \bl) \mid D]  \right)^2 \\
&= \E \left( \frac{1}{n} \sum_{i=1}^n  \E[ \phi(y_o, \bm{x}_o^\top \bli)-\phi(y_o, \bm{x}_o^\top \bl) \mid D]   \right)^2
\\
&= \E \left( \frac{1}{n} \sum_{i=1}^n  \E[ \pd(y_o, t\bm{x}_o^\top \bl + (1-t)\bm{x}_o^\top \bli )  \bm{x}_o^\top (\bli-\bl) \mid D]   \right)^2 \\
&\leq \E \left( \frac{1}{n} \sum_{i=1}^n  \sqrt{\E[ \pd(y_o, t\bm{x}_o^\top \bl + (1-t)\bm{x}_o^\top \bli )^2 \mid D]} \sqrt{  \E[(\bm{x}_o^\top (\bli-\bl))^2 ) \mid D]}   \right)^2
 \end{align*}
  \begin{align*}
&\leq c_1^2 \E \left( \frac{1}{n} \sum_{i=1}^n   \sqrt{  \E[(\bm{x}_o^\top (\bli-\bl))^2 ) \mid D]}   \right)^2
\\
&\leq c_1^2 \E \left( \frac{1}{n} \sum_{i=1}^n   \sqrt{   (\bli-\bl)^\top \bm{\Sigma} (\bli-\bl)}   \right)^2
\\
&\leq c_1^2 \E \left( \frac{1}{n} \sum_{i=1}^n   \sqrt{\frac{\rho}{p} }  \| \bli-\bl \|_2 \right)^2
\\
&\leq c_1^2 \frac{\rho}{p}   \left( \frac{1}{n}    \E\| \blone-\bl \|_2^2   + \frac{n-1}{n} \E\| \blone-\bl \|_2\| \bltwo-\bl \|_2  \right)
\\
&\leq c_1^2 \frac{\rho}{p}  \E\norm{\blone-\bl}_2^2
\\
&\leq c_1^2 c_0^2 \frac{\rho}{p \nu^2}  \E\| \bm{x}_1 \|_2^2,
\end{align*}
where the last inequality is due to Lemma \ref{lem:perturb_i}. Using the fact that $\delta=n/p$, $\bm{x}_1 \sim N(\bm{0}, \bm{\Sigma})$ and $\sigma_{\max}(\bm{\Sigma}) = \rho/p$, we get
\begin{align*}
\E \Bigl( \frac{1}{n} \sum_{i=1}^n  \E\bigl[ \phi(y_i, \bm{x}_i^\top \bli) \mid D_{/i}\bigr]   - \E\bigl[\phi(y_o, \bm{x}_o^\top \bl) \mid D\bigr]  \Bigr)^2 &\leq \frac{1}{n}\left(\frac{c_0 c_1 \rho \delta^{1/2}}{\nu}\right)^2.
\end{align*}
\end{proof}

\begin{lemma}\label{lem:perturb_i} If both the loss function and the regularizer are twice differentiable, then for all $i = 1, \dotsc, n$:
\begin{align*}
\| \bli - \bl \|_2^2 &\leq  \left( \frac{\ld_i(\bl)}{ \inf_{t\in[0,1]}\sigma_{\min}( \bm{J}_{/i} (t\bl + (1-t) \bli) ) } \right)^2  \| \bm{x_i} \|_2^2,
\\
\| \bli - \bl_{/ij} \|_2^2 &\leq  \left( \frac{\ld_j(\bli)}{ \inf_{t\in[0,1]}\sigma_{\min}( \bm{J}_{/ij} (t\bli + (1-t) \bl_{/ij}) ) } \right)^2  \| \bm{x_j} \|_2^2.
\end{align*}
\end{lemma}

\begin{proof} [Proof of Lemma \ref{lem:perturb_i}]
The leave-one-out estimate, $\bli = \bl + \bm{\Delta}_{/i}$,  satisfies $\bm{f}_{/i} (\bm{\Delta}_{/ i}) = 0$. The multivariate mean-value Theorem yields
\begin{eqnarray}
0 &=& \bm{f}_{/i} (\bl + \bm{\Delta}_{/i})  = \bm{f}_{/i} (\bl) + \left(\int_0^1 \bm{J}_{/i} (\bl + t \bm{\Delta}_{/i}) dt \right)\bm{\Delta}_{/i}
\end{eqnarray}
where the Jacobean is
\begin{eqnarray}
\bm{J}_{/i} (\bm{\theta}) &=&  \lambda \bm{ \nabla^2 r}(\bm{\theta})  +  \XI^\top \diag[\bm{\ldd}_{/ i}( \bm{\theta})] \XI.
\end{eqnarray}
Moreover, $\bl$ satisfies
\begin{eqnarray*}
0 &=& \lambda \bm{ \nabla r}(\bl)+ \bm{X}^\top\bm{\ld} (\bl)  = \bm{f}_{/i} (\bl) + \ld_i(\bl) \bm{x_i}.
\end{eqnarray*}
We get
\begin{eqnarray*}
\ld_i(\bl) \bm{x_i} &=& \left(\int_0^1 \bm{J}_{/i} (\bl + t \bm{\Delta}_{/i}) dt \right)\bm{\Delta}_{/i},
\end{eqnarray*}
leading to
\begin{eqnarray*}
\bm{\Delta}_{/i} &=&\ld_i(\bl) \left(\int_0^1 \bm{J}_{/i} (\bl + t \bm{\Delta}_{/i}) dt \right)^{-1}\bm{x_i},
\end{eqnarray*}
and
\begin{eqnarray*}
\| \bm{\Delta}_{/i} \|_2^2 &\leq& \left( \frac{\ld_i(\bl)}{ \inf_{t\in[0,1]}\sigma_{\min}( \bm{J}_{/i} (t\bl + (1-t) \bli) ) } \right)^2  \| \bm{x_i} \|_2^2.
\end{eqnarray*}
Likewise,
\begin{eqnarray*}
\| \bli - \bl_{/ij} \|_2^2 &\leq&  \left( \frac{\ld_j(\bli)}{ \inf_{t\in[0,1]}\sigma_{\min}( \bm{J}_{/ij} (t\bli + (1-t) \bl_{/ij}) ) } \right)^2  \| \bm{x_j} \|_2^2.
\end{eqnarray*}
\end{proof}

\section{Proof of Corollary \ref{ex:logistic}}\label{sec:proof:ex:logistic}
We would like to use Theorem \ref{th:mse} to prove this corollary. Toward this goal, we first have to prove that Assumptions \ref{A1}, \ref{A2}, and \ref{A3} hold, and that $\var[ \phi(y_o, \bm{x}_o^\top \bli) \mid D_{/i}]$ is bounded. Given the fact that $\bm{x}_i$ is $N(0, \bm{\Sigma})$, Assumption \ref{A1} holds. As we discussed in Example 1, Assumption \ref{A2} holds as well with $\nu=\lambda$. Finally, given that the regularizer is ridge Assumption \ref{A3} holds too. Hence, the only remaining step is to check the boundedness of $\var[ \phi(y_o, \bm{x}_o^\top \bli) \mid D_{/i}]$. In the rest of the proof we aim to prove that
\begin{equation}\label{eq:logvar}
 \var[ \phi(y_o, \bm{x}_o^\top \bli) \mid D_{/i}]  \leq 6 + \frac{5 \rho \delta }{\lambda}.
\end{equation}  
Note that
\begin{align*}
 &\var[ \phi(y_o, \bm{x}_o^\top \bli) \mid D_{/i}] \\
 &\quad=  \var[ -y_o \bm{x}_o^\top \bli + \log(1 + e^{\bm{x}_o^\top \bli}) \mid D_{/i}]
 \\
 &\quad\leq
 \E[(\bm{x}_o^\top \bli)^2  \mid D_{/i}] + \E[ \log^2(1 + e^{\bm{x}_o^\top \bli}) \mid D_{/i}]  \\
 &\qquad + 2 \sqrt{\E[ (\bm{x}_o^\top \bli)^2  \mid D_{/i}]  \E[  \log^2(1 + e^{\bm{x}_o^\top \bli}) \mid D_{/i}] }
  \\
 &\quad\leq \E[(\bm{x}_o^\top \bli)^2 \mid D_{/i}]+\E[ (1 + |\bm{x}_o^\top \bli|)^2 \mid D_{/i}] \\
 &\qquad + 2 \sqrt{\E[ (\bm{x}_o^\top \bli)^2  \mid D_{/i}]  \E[ (1 + |\bm{x}_o^\top \bli|)^2 \mid D_{/i}] },
 \end{align*}
 where to obtain the last inequality we have used $\log (1+ e^z) \leq 1+ |z|$. Furthermore, since $(1+ |z|)^2 \leq 2 +2z^2$, we have
 \begin{align}
 & \var[ \phi(y_o, \bm{x}_o^\top \bli) \mid D_{/i}] \nonumber \\
 & \quad \leq 2+ 3 \E[(\bm{x}_o^\top \bli)^2 \mid D_{/i}] \nonumber \\
 & \qquad + 4 (1+\E[(\bm{x}_o^\top \bli)^2 \mid D_{/i}]) \nonumber \\
 & \quad = 6 + 7 \E[(\bm{x}_o^\top \bli)^2 \mid D_{/i}] \nonumber \\
 & \quad = 6 + 7 \bli^\top \bm{\Sigma} \bli \leq 6 + \frac{7 \rho}{p} \bli^\top \bli. 
 \end{align}
Comparing $\bli$ with the zero estimator yields, $n \log 2 \geq \bm{y_{/i}}^\top \bm{X_{/i}} \bli + \bm{1}^\top \log(1+e^{\bm{X_{/i}} \bli}) + \lambda \|\bli \|_2^2/2$. Since $\log(1+e^z) - z \geq 0$ for any $z$, we get $\lambda \norm{\bli}_2^2 \leq n \log 2$. Therefore, we can say that for ridge regularized logistic regression
\begin{equation*}
 \var[ \phi(y_o, \bm{x}_o^\top \bli)| D_{/i}]  \leq 6 + \frac{5 \rho \delta }{\lambda}.
\end{equation*}  
To summarize, using the bound above and Theorem \ref{th:mse} (with $C_b=(\frac{c_0 c_1 \rho \delta^{1/2}}{\nu})^2$) , for ridge regularized logistic regression, we conclude that 
\begin{eqnarray}
C_v &=&  \E \var[ \phi(y_o, \bm{x}_o^\top \blone) \mid D_{/1}]  + 2 C_b +
2 C_b^{1/2} \sqrt{\E \var[ \phi(y_o, \bm{x}_o^\top \blone) \mid D_{/1}]  +  C_b} \nonumber
\\
&=&6 + \frac{5 \rho \delta }{\lambda} + 2\left(\frac{c_0 c_1 \rho \delta^{1/2}}{\nu}\right)^2 + 2\left(\frac{c_0 c_1 \rho \delta^{1/2}}{\nu}\right)\sqrt{6 + \frac{5 \rho \delta }{\lambda} + \left(\frac{c_0 c_1 \rho \delta^{1/2}}{\nu}\right)^2}, \nonumber
\\
&=&6 + \frac{5 \rho \delta }{\lambda} + 2\left(\frac{4 \rho \delta^{1/2}}{\lambda}\right)^2 + 2\left(\frac{4 \rho \delta^{1/2}}{\lambda}\right)\sqrt{6 + \frac{5 \rho \delta }{\lambda} + \left(\frac{4 \rho \delta^{1/2}}{\lambda}\right)^2}
\end{eqnarray}
where the last equation is due to $c_0=c_1=2$ (shown in Example 1) and $\nu=\lambda$.

%%------------------------%%--------------------------------%%
%%------------------------%%--------------------------------%%

\section{Proof of Corollary \ref{example:psedoHuber}}\label{sec:proof:example:psedoHuber}
We would like to use Theorem \ref{th:mse} to prove this corollary. Toward this goal we have to confirm Assumptions \ref{A1}, \ref{A2}, and \ref{A3} and prove the boundedness of $\E \left(\Var[f_H(y_0, \xv_0^\top \bli) \mid D_{/i}]\right)$. Assumption \ref{A1} is already assumed in the corollary. Assumption \ref{A2} is also confirmed in Example 2 so that $c_0=c_1=\gamma$. Since the regularizer is assumed to be strongly convex, Assumption \ref{A3} is also automatically satisfied with $\nu=\nu_r$. Hence, the only remaining step is to obtain an upper bound for $\E \left(\Var[f_H(y_0, \xv_0^\top \bli) \mid D_{/i}]\right)$. In the rest of the proof we prove that
\begin{align*}
\MoveEqLeft{\E \left(\Var[f_H(y_0, \xv_0^\top \bli) \mid D_{/i}]\right)}
\leq 2 \left(\gamma^4 + \frac{\rho \gamma^3 \delta}{\nu_r} (\sigma_\epsilon + \sqrt{\rho b}) + \gamma^2 (\rho b + \sigma_\epsilon^2)\right). 
\end{align*}
We have
\begin{eqnarray*}
 \lefteqn{ \Var[f_H(y_o, \xv_o^\top \bli) \mid D_{/i}] = \gamma^4 \Var\Bigl[ \left( \Bigl\{1 + \bigl(\frac{y_o - \xv_o^\top \bli}{\gamma}\bigr)^2\Bigr\}^{1/2} -1 \right)^2 \mid \bli \Bigr]} 
 \\
  &\leq& \gamma^4 \E \Bigl[ \left( \Bigl\{1 + \bigl(\frac{y_o - \xv_o^\top \bli}{\gamma}\bigr)^2\Bigr\}^{1/2} -1 \right)^2 \mid \bli \Bigr] \\
  &\leq& \gamma^4 \E \Bigl[2 + \frac{\abs{y_o - \xv_o^\top \bli}^2}{\gamma^2} \mid \bli \Bigl] \\
  &\leq&2 \gamma^4+ \gamma^2 \E[|y_o - \xv_o^\top \bli|^2 \mid \bli]
  \\
  &\leq& 2 \gamma^4+2 \gamma^2 \E[y_o^2+(\xv_o^\top \bli)^2 \mid \bli] \\
  &\leq& 2 \gamma^4+ 2\gamma^2 (\E[y_o^2\mid \bli] + \frac{\rho}{p} \norm{\bli}_2^2).
\end{eqnarray*}
Furthermore, we have that $\E[y_o^2\mid \bli] \leq \frac{ \rho {\bm{\beta^*}}^\top \bm{\beta^*}}{p} + \Var[\epsilon_0] \leq \rho b +\sigma_\epsilon^2$. Additionally, note that using the strong convexity of the regularizer, and by comparing the value of $\sum_{j\neq i} f_H(y_j- \bm{x}_j^\top \bm{\beta}) + \lambda r(\bm{\beta})$ at $\bli$ and $\mathbf{0}$, we have that $\norm{\bli}_2^2 \leq \nu_r^{-1}\sum_{j \neq i} \gamma |y_j| \leq \nu_r^{-1} \gamma \norm{\yv}_1$.\footnote{We have used the fact that $f_H(a) \leq \gamma|a|$.}
Therefore,
\begin{equation*}
    \E \frac{\rho}{p}\norm{\bli}_2^2 
    \leq 
    \frac{\rho \gamma}{\nu_r} \frac{1}{p} \E \|\yv \|_1 = \frac{\rho \gamma \delta}{\nu_r} \E \abs{y_1}.
\end{equation*}
We may bound this quantity explicitly in terms of the covariance of $\xv$:
\begin{align*}
    \E \abs{y_1} &\leq \E \abs{\epsilon_1} + \abs{\xv_1^\top \betav^*} \leq \sqrt{\E \epsilon_1^2} + \sqrt{\E (\xv_1^\top \betav^*)^2} \nonumber \\ &\leq \sigma_\epsilon + \sqrt{\frac{\rho}{p} \norm{\betav^*}_2^2 } \leq \sigma_\epsilon + \sqrt{\rho b}.
\end{align*}
Hence,
\begin{align*}
\E \left(\Var[f_H(y_0, \xv_0^\top \bli) \mid D_{/i}]\right)\leq 2 \left(\gamma^4 + \frac{\rho \gamma^3 \delta}{\nu_r} (\sigma_\epsilon + \sqrt{\rho b}) + \gamma^2 (\rho b + \sigma_\epsilon^2)\right). 
\end{align*}
To summarize, using the bound above and Theorem \ref{th:mse} (with $C_b=(\frac{c_0 c_1 \rho \delta^{1/2}}{\nu})^2$), we conclude that 
\begin{eqnarray}
C_v &=&  \E \var[ \phi(y_o, \bm{x}_o^\top \blone) \mid D_{/1}]  + 2 C_b +
2 C_b^{1/2} \sqrt{\E \var[ \phi(y_o, \bm{x}_o^\top \blone) \mid D_{/1}]  +  C_b} \nonumber
\\
&=& 2 \left(\gamma^4 + \frac{\rho \gamma^3 \delta}{\nu_r} (\sigma_\epsilon + \sqrt{\rho b}) + \gamma^2 (\rho b + \sigma_\epsilon^2)\right) + 2\left(\frac{c_0 c_1 \rho \delta^{1/2}}{\nu}\right)^2 \nonumber \\
&+& 2\left(\frac{c_0 c_1 \rho \delta^{1/2}}{\nu}\right)\sqrt{ 2 \left(\gamma^4 + \frac{\rho \gamma^3 \delta}{\nu_r} (\sigma_\epsilon + \sqrt{\rho b}) + \gamma^2 (\rho b + \sigma_\epsilon^2)\right) + \left(\frac{c_0 c_1 \rho \delta^{1/2}}{\nu}\right)^2}  \nonumber
\\
&=&
2 \left(\gamma^4 + \frac{\rho \gamma^3 \delta}{\nu_r} (\sigma_\epsilon + \sqrt{\rho b}) + \gamma^2 (\rho b + \sigma_\epsilon^2)\right) + 2\left(\frac{\gamma^2 \rho \delta^{1/2}}{\nu_r}\right)^2 \nonumber \\
&+& 2\left(\frac{\gamma^2 \rho \delta^{1/2}}{\nu_r}\right)\sqrt{ 2 \left(\gamma^4 + \frac{\rho \gamma^3 \delta}{\nu_r} (\sigma_\epsilon + \sqrt{\rho b}) + \gamma^2 (\rho b + \sigma_\epsilon^2)\right) + \left(\frac{\gamma^2 \rho \delta^{1/2}}{\nu_r}\right)^2}
\nonumber
\end{eqnarray}
where the last equation is due to $c_0=c_1=\gamma$ (shown in Example 2) and $\nu=\nu_r$.

\section{Proof of Example \ref{ex:min_eig_oversamp}} \label{sec:proof:ex:min_eig_oversamp}

It is straightforward to check that, for any $i$, we have:
\[
\inf_{t \in [0,1]}  \sigma_{\min}( \bm{A}_{t,/ i })  \geq c  \sigma_{\min} (  \XI^\top \XI).
\]
This implies that:
\[
 \mathbb{E} \Big(\inf_{t \in [0,1]}  \sigma_{\min}( \bm{A}_{t,/ i })  \Big)^{-8}\leq \frac{1}{c^8}  \mathbb{E} \sigma_{\min}^{-8} (  \XI^\top \XI).
\]
Define the vectors $\bm{z}_i = \bm{\Sigma}^{-\frac{1}{2}} \bm{x}_i$. Hence, $\bm{z}_i \sim N(0, \bm{I})$.
Furthermore, define the matrix $\bm{Z}$ as the matrix that has $\bm{z}_i$ as its rows.
It is straightforward to check that:
\begin{equation*}
\sigma_{\min} (\XI^\top \XI) = \sigma_{\min} (\sum_{j\neq i} \bm{x}_j \bm{x}_j^\top)
\geq \frac{\rho}{p} \sigma_{\min} (\sum_{j\neq i} \bm{z}_j \bm{z}_j^\top) = \rho \delta \sigma_{\min}\left(\frac{\bm{Z}_{/ i}^\top \bm{Z}_{/i}}{n}\right).
\end{equation*}
The fact that the quantity $\mathbb{E} \sigma_{\min}^{-8} \left( \frac{\bm{Z}_{/ i}^\top \bm{Z}_{/i}}{n}\right)$ is lower bounded by a constant for large values of $n,p$ when $n/p = \delta>1$ is proved in \cite{XMRD19}. See Lemma \ref{lem:socomp} in the supplementary material.

\section{Proof of Theorem \ref{th:mse2}}\label{sec:GenericTheorem}
The proof of this result is very similar to the proof of Theorem \ref{th:mse}. Hence, instead of rewriting the proof, we only emphasize on the differences between the proofs of Theorems \ref{th:mse} and \ref{th:mse2}. The strategy of the proof is exactly the same. We break the error between $\lo$ and $\extra$ into $V_1$ and $V_2$ and try to bound the second moments of these quantities. 
The following lemma obtains an upper bound for the second moment of $V_1$. 
\begin{lemma}
Under the assumptions of Theorem \ref{th:mse2} we have
\begin{equation*}
\E \biggl(\frac{1}{n} \sum_{i=1}^n \phi(y_i, \bm{x}_i^\top \bli)  - \frac{1}{n} \sum_{i=1}^n  \E[ \phi(y_i, \bm{x}_i^\top \bli) \mid D_{/i}] \biggr)^2
\leq
\frac{1}{n}\biggl(   \E \var[ \phi(y_o, \bm{x}_o^\top \blone) \mid
D_{/1}]  +
\tilde{c}_0\tilde{c}_1^2 \rho \delta \tilde{v} c_4 \biggr)
\end{equation*}
\end{lemma}
\begin{proof}
The proof of this lemma is similar to the proof of Lemma \ref{lem:v1}. All the steps are exactly the same up to the point that is proved:
 \begin{align*}
  &\E \biggl[\Bigl( \phi(y_1, \bm{x}_1^\top \blone)  -  \E[ \phi(y_1, \bm{x}_1^\top \blone) \mid D_{/1}] \Bigr) \Bigl( \phi(y_2, \bm{x}_2^\top \bltwo)  -  \E[ \phi(y_2, \bm{x}_2^\top \bltwo) \mid D_{/2}] \Bigr)\biggr] \\
  &\leq \E \Bigl(  \E  \bigl[ \pd(y_1, t\bm{x}_1^\top \blone + (1-t)\bm{x}_1^\top \blonetwo)^2 \mid D_{/1}  \bigr]
    \E      \bigl[ \bigl( \bm{x}_1^\top (\blone-\blonetwo) \bigr)^2  \mid D_{/1}  \bigr]
    \Bigr)  \\
  &\leq \tilde{c}_1^2 \E \Bigl( \E\bigl[ \bigl(\bm{x}_1^\top (\blone-\blonetwo)\bigr)^2  \mid D_1 \bigr] \Bigr)
\\
    &= \tilde{c}_1^2 \E \Bigl( \bigl(\bm{x}_1^\top (\blone-\blonetwo) \bigr)^2  \Bigr)  \\
    &= \tilde{c}_1^2 \E \left( (\blone-\blonetwo)^\top \bm{\Sigma} (\blone-\blonetwo)  \right)  \\
    &\leq \tilde{c}_1^2 \frac{\rho}{p} \E \norm{ \blone-\blonetwo }_2^2.
\end{align*}
However, the way we would like to bound $\E \norm{ \blone-\blonetwo }_2^2$ here is slightly different from the approach used in the proof of Lemma \ref{lem:v1}.
According to Lemma \ref{lem:perturb_i} we have:
\[
\| \blone - \blonetwo \|_2^2 \leq  \left( \frac{\ld_j(\blone)}{ \inf_{t\in[0,1]}\sigma_{\min}( \bm{J}_{/1,2} (t\blone + (1-t) \blonetwo) ) } \right)^2  \| \bm{x}_2 \|_2^2.
\]
Hence, by using the Cauchy-Schwarz inequality twice we obtain:
\[
\E \| \blone - \blonetwo \|_2^2 \leq   \E |(\ld_j(\blone))|^8 \E \left[ \frac{1}{(\inf_{t\in[0,1]}\sigma_{\min}( \bm{J}_{/1,2} (t\blone + (1-t) \blonetwo) ))^8 } \right] \E\| \bm{x}_2 \|_2^4 \leq \tilde{c}_0 \tilde{\nu} c_4.
\]

 Using the fact that $\delta=n/p$, $\bm{x}_1 \sim N(\bm{0}, \bm{\Sigma})$ and $\sigma_{\max}(\bm{\Sigma}) = \rho/p$, we get
\begin{equation*}
  \E  \Bigl[\Bigl( \phi(y_1, \bm{x}_1^\top \blone)  -  \E[ \phi(y_1, \bm{x}_1^\top \blone) \mid D_{/1}]\Bigr) \Bigl( \phi(y_2, \bm{x}_2^\top \bltwo)  -  \E[ \phi(y_2, \bm{x}_2^\top \bltwo)\mid D_{/2}]\Bigr)\Bigr] \leq \frac{1}{n} ( \tilde{c}_0\tilde{c}_1^2 \rho \delta \tilde{v} c_4).
\end{equation*}
\end{proof}

The second Lemma aims to obtain an upper bound for the second moment of $V_2$. This corresponds to Lemma \ref{lem:v2} in the proof of Theorem \ref{th:mse}.

\begin{lemma}\label{lem:v2prime}
Under the assumptions of Theorem \ref{th:mse2}, we have
\begin{equation*}
\E \left( \frac{1}{n} \sum_{i=1}^n  \E[ \phi(y_i, \bm{x}_i^\top \bli) \mid D_{/i}]    -   \E [\phi(y_o, \bm{x}_o^\top \bl) \mid D]  \right)^2
\leq c_1^2 \frac{\rho}{p}  \E\norm{\blone-\bl}_2^2  \leq c_1^2 \rho \delta_0 \tilde{c}_0 \tilde{v} c_4. 
\end{equation*}
\end{lemma}
\begin{proof}
Again the proof follows very similar to the steps as the proof of Lemma \ref{lem:v2}. In fact, we follow exactly the same steps until it is proved that 
\begin{equation*}
\E \left( \frac{1}{n} \sum_{i=1}^n  \E[ \phi(y_i, \bm{x}_i^\top \bli) \mid D_{/i}]   - \E [\phi(y_o, \bm{x}_o^\top \bl) \mid D]  \right)^2
\leq c_1^2 \frac{\rho}{p}  \E\norm{\blone-\bl}_2^2.
\end{equation*}
Then, in order to bound $\E\norm{\blone-\bl}_2^2$ we use a slightly different strategy. According to Lemma \ref{lem:perturb_i} we have
\begin{equation*}
\| \blone - \bl \|_2^2
\leq  \left( \frac{\ld_1(\bl)}{ \inf_{t\in[0,1]}\sigma_{\min}( \bm{J}_{/1} (t\bl + (1-t) \blone) ) } \right)^2  \| \bm{x}_1 \|_2^2.
\end{equation*}
Hence, by using Cauchy-Schwarz inequality we have:
\[
\frac{1}{p} \E \left( \| \blone - \bl \|_2^2 \right) \leq \delta_o \tilde{c}_0 \tilde{v} c_4,
\]
from which we deduce:
\begin{equation*}
 \E \left( \frac{1}{n} \sum_{i=1}^n  \E[ \phi(y_i, \bm{x}_i^\top \bli) \mid D_{/i}]   - \E [\phi(y_o, \bm{x}_o^\top \bl) \mid D]  \right)^2
\leq c_1^2 \frac{\rho}{p}  \E\norm{\blone-\bl}_2^2  \leq c_1^2 \rho \delta_0 \tilde{c}_0 \tilde{v} c_4. 
\end{equation*}
\end{proof}

%--------------------------------------------------------------------%--------------------------------------------------------------------

%--------------------------------%--------------------------------%--------------------------------%--------------------------------
%--------------------------------%--------------------------------%--------------------------------%--------------------------------
%--------------------------------%--------------------------------%--------------------------------%--------------------------------
%--------------------------------------------------------------------%--------------------------------------------------------------------

%%------------------------%%--------------------------------%%
%%------------------------%%--------------------------------%%

\section{Proof of Corollary \ref{ex:squareLoss}} \label{ssec:proof:ExampleElasticNet}
As is clear, we would like to use Theorem \ref{th:mse2} to prove our claim. Toward this goal, we have to prove that Assumptions \ref{Aprime1}, \ref{Aprime2}, and \ref{Aprime3} hold. Furthermore, we have to obtain an upper bound for the constant $\tilde{C}_v$, which in turn requires us to bound $\E \var[ \phi(y_o, \bm{x}_o^\top \blone) \mid D_{/1}]$. 
Given that $\bm{x}_i$ is Gaussian, Assumption \ref{Aprime1} is automatically satisfied. Furthermore, since the regularizer is elastic-net, it is straightforward to prove Assumption \ref{Aprime3}. To see this, first note that, for all $i, j$, we have almost surely:
\begin{align}
\bm{A}_{t,/ i} &\triangleq \XI^\top \diag[\bm{\ldd}_{/ i}( t \bli +(1-t)\bl)] \XI + \lambda \bm{\nabla^2 r} (t \bli +(1-t)\bl), \nonumber \\
\bm{A}_{t,/ i,j} &\triangleq \bm{X}_{/ij}^\top \diag[\bm{\ldd}_{/ ij}( t \bl_{/ij} +(1-t)\bli)] \bm{X}_{/ij} + \lambda \bm{\nabla^2 r}(t \bl_{/ij} +(1-t)\bli), \nonumber
\end{align}
where $r(\beta) = \gamma \beta^2 + (1-\gamma) r^{\alpha}(\beta)$. Hence, it is straightforward to see that 
\begin{align}
\sigma_{\min} (\bm{A}_{t,/ i}) & \geq \lambda \gamma, \nonumber \\
\sigma_{\min} (\bm{A}_{t,/ i,j} )& \geq \lambda \gamma. \nonumber
\end{align}
Hence, the only remaining steps are to prove Assumption \ref{Aprime2} and bound the term $\E \var[ \phi(y_o, \bm{x}_o^\top \blone) \mid D_{/1}]$. Given that $\phi(y_o, \bm{x}_o^\top \blone) = \frac{1}{2} (y_o-\bm{x}_o^\top \blone)^2$, we have
\[
\var[ \phi(y_o, \bm{x}_o^\top \blone) \mid D_{/1}] \leq \frac{1}{4} \E [ (y_o-\bm{x}_o^\top \blone)^4 \mid D_{/1}].
\]
Hence,
\[
\E \var[ \phi(y_o, \bm{x}_o^\top \blone) \mid D_{/1}] \leq \frac{1}{4} \E [ (y_o-\bm{x}_o^\top \blone)^4] \leq \frac{1}{4} \left(\E [ (y_o-\bm{x}_o^\top \blone)^8] \right)^{0.5}.
\]
Hence, if we prove Assumption \ref{Aprime2}, we have also proved that 
\[
\E \var[ \phi(y_o, \bm{x}_o^\top \blone) \mid D_{/1}] \leq \frac{1}{4} \left(\E [ (y_o-\bm{x}_o^\top \blone)^8] \right)^{0.5} \leq \frac{\tilde{c}_0^{0.5}}{4}. 
\]
In the rest of this section, we focus on the proof of Assumption \ref{Aprime2}. Note that $\ld(y, \bm{x}_i^\top \bl)= y_i -\bm{x}_i^\top \bl$. Under these assumptions, we prove that there exists a fixed number $\tilde{c}_0$ such that $\mathbb{E} (y_i -\bm{x}_i^\top \bl)^8 \leq { \tilde c_0}$, and  $\mathbb{E} (y_o-  \bm{x}_o^\top \bli)^8  \leq { \tilde c_0}$. 

Consider the following definitions:
\begin{eqnarray}
\bl = \arg\min_{\bm{\beta}} f(\bm{\beta}) =  \arg\min_{\bm{\beta}}  \sum_{j=1}^n \frac{(y_j - \bm{x}_j^\top \bm{\beta})^2}{{2}} + \lambda \sum_{i=1}^p r(\beta_i), \nonumber \\
\bli = \arg\min_{\bm{\beta}} f_{\slash i}(\bm{\beta}) =  \arg\min_{\bm{\beta}}  \sum_{j=1, j \neq i}^n \frac{(y_j - \bm{x}_j^\top \bm{\beta})^2}{{2}} + \lambda \sum_{i=1}^p r(\beta_i)
\end{eqnarray}
Furthermore, define $r_{0.5} (\beta) = \frac{\gamma}{2} \beta^2 + (1-\gamma) r^\alpha (\beta)$. Our optimization problem can be written as 
\[
\bl = \arg\min_{\bm{\beta}} f(\bm{\beta}) =  \arg\min_{\bm{\beta}}  \sum_{j=1}^n \frac{(y_j - \bm{x}_j^\top \bm{\beta})^2}{{2}} + \lambda \sum_{i=1}^p r_{0.5}(\beta_i) + \frac{\lambda \gamma}{2} \sum_{i=1}^p \beta_i^2.
\]
Since $\bm{y}=\bm{X\beta^*} + \bm{\e}$, where $\bm{\e} \sim \N(0,\sigma_{\e}^2 \bm{I} )$, the optimality conditions yield 
\begin{eqnarray*}
\bm{X}^\top (\bm{X}\bl-\bm{y})+{\lambda \gamma} \bl+ \lambda \bm{\rd}_{0.5}(\bl) = 0.
\end{eqnarray*}
Hence,
\[
 \bl = (\bm{X}^\top \bm{X}+ \lambda \gamma \bm{I})^{-1}\bm{X}^{\top}\bm{y} - \lambda(\bm{X}^\top \bm{X}+ \lambda \gamma \bm{I})^{-1} \bm{\rd}_{0.5}(\bl).
\]
It is then straightforward to prove that
\begin{eqnarray}
\bm{y}- \bm{X}\bl &=&  (\bm{I}- \bm{X}(\bm{X}^\top \bm{X}+ \lambda \gamma \bm{I})^{-1} \bm{X}^{\top})\bm{y} {+} \lambda    \bm{X}(\bm{X}^\top \bm{X}+ \lambda \gamma \bm{I})^{-1}\bm{\rd}_{0.5}(\bl) \nonumber  \\
&=&  (\bm{I}- \bm{X}(\bm{X}^\top \bm{X}+ \lambda \gamma \bm{I})^{-1} \bm{X}^{\top})\bm{X \beta^*} +  (\bm{I}- \bm{X}(\bm{X}^\top \bm{X}+ \lambda \gamma \bm{I})^{-1} \bm{X}^{\top})\bm{\epsilon} \nonumber \\
&& + {\lambda} \bm{X}(\bm{X}^\top \bm{X}+ \lambda \gamma \bm{I})^{-1}\bm{\rd}_{0.5}(\bl).\label{eq:three}
\end{eqnarray}

Our goal is to show that all the ``finite" moments of the elements of $\bm{y}_i- \bm{x}_i^\top \bl$, including the $8^{\rm th}$ moment required in our example, are $O(1)$. From \eqref{eq:three} we have
\begin{eqnarray}
\mathbb{E} |\bm{y}_i- \bm{x}_i^\top \bl|^k &\leq& 3^{k-1} \Big( \mathbb{E} (1- \bm{x}_i^\top  (\bm{X}(\bm{X}^\top \bm{X}+\lambda \gamma \bm{I})^{-1} \bm{X}^{\top})\bm{X \beta^*})^k + \mathbb{E}|1- \bm{x}_i^\top(\bm{X}^\top \bm{X} \lambda \gamma \bm{I})^{-1} \bm{X}^{\top}) \bm{\epsilon}|^k \nonumber \\
&&+ \lambda^k \mathbb{E} |\bm{x}_i^\top (\bm{X}^\top \bm{X}+ \lambda \gamma \bm{I})^{-1} \bm{\rd}_{0.5}(\bl)|^k \Big)
\end{eqnarray}
Hence, we bound each of the above three terms separately in the following lemmas:

\begin{lemma} \label{lem:l1}
Under the assumptions of Example \ref{ex:squareLoss} we have 
\[
\mathbb{E} (1- \bm{x}_i^\top  (\bm{X}(\bm{X}^\top \bm{X}+\lambda \gamma \bm{I})^{-1} \bm{X}^{\top})\bm{X \beta^*})^k \leq \left(\frac{\rho}{p \lambda^2 \gamma^2} \|\bm{\beta^*}\|_2\right)^{2k} k!!.
\]

\end{lemma}
\begin{proof}
First note that
\begin{equation}\label{eq:biasfirstexp1}
 (\bm{I}- \bm{X}(\bm{X}^\top \bm{X}+ \lambda \gamma \bm{I})^{-1} \bm{X}^{\top}) \bm{X}\bm{\beta^*} =  \lambda \gamma \bm{X}(\bm{X}^\top \bm{X}+ \lambda \gamma \bm{I})^{-1} \bm{\beta^*}. 
\end{equation}
Hence,
\[
1- \bm{x}_i^\top  (\bm{X}(\bm{X}^\top \bm{X}+\lambda \gamma \bm{I})^{-1} \bm{X}^{\top})\bm{X \beta^*} = \lambda \gamma \bm{x}_i^\top (\bm{X}^\top \bm{X}+ \lambda \gamma \bm{I})^{-1} \bm{\beta^*}. 
\]
Define $\bm{D}_i = (\XI^\top \XI +\lambda \gamma \bm{I})^{-1}$. According to the matrix inversion lemma we have
\begin{eqnarray}\label{eq:upperbias1}
\bm{x}_i^\top (\bm{X}^\top \bm{X}+ \lambda \gamma \bm{I})^{-1} \bm{\beta^*} = \bm{x}_i^\top \bm{D}_i \bm{\beta^*} - \frac{{\bm x}_i^\top \bm{D}_i \bm{x}_i \bm{x}^{\top}_i \bm{D}_i \bm{\beta^*} }{1+ \bm{x}_i^\top \bm{D}_i \bm{x_i} } = \frac{ \bm{x}^{\top}_i \bm{D}_i \bm{\beta^*} }{1+ \bm{x}_i^\top \bm{D}_i \bm{x_i} }.
\end{eqnarray}
Note that conditioned on $\XI$ the distribution of $\bm{x}_i^\top \bm{D}_i \bm{\beta^*}$ is a zero mean Gaussian random variable with variance $v_i =  \| \bm{\Sigma}^{1/2} \bm{D}_i \bm{\beta^*}\|_2^2 \leq \frac{\rho}{p \lambda^2 \gamma^2} \|\bm{\beta^*}\|_2^2$. Hence, \eqref{eq:upperbias1} and the moments of a Gaussian random variable (see Lemma \ref{lem:momenet:Gauss}) lead to 
\begin{eqnarray}
\mathbb{E} (|\bm{x}_i^\top (\bm{X}^\top \bm{X}+ \lambda \gamma \bm{I})^{-1} \bm{\beta^*}|^k \ | \ \XI) \leq \nu_i^k (k-1)!!.
\end{eqnarray}
 Hence, by the law of iterated expectation, we obtain
\[
\mathbb{E} (|\bm{x}_i^\top (\bm{X}^\top \bm{X}+ \lambda \gamma \bm{I})^{-1} \bm{\beta^*}|^k) \leq \nu_i^k (k-1)!! \leq \left(\frac{ \rho}{p \lambda^2 \gamma^2} \|\bm{\beta^*}\|_2\right)^{2k} k!!. 
\]
\end{proof}

\begin{lemma}\label{lem:l2}
Under the assumptions of Example \ref{ex:squareLoss}, if $\bm{\e} \sim \N(0, \sigma_{\e}^2 \bm{I})$, then
\[
\mathbb{E}|1- \bm{x}_i^\top(\bm{X}^\top \bm{X}+ \lambda \gamma \bm{I})^{-1} \bm{X}^{\top}) \bm{\epsilon}|^k \leq \sigma_\epsilon^k (k-1)!!.
\]
\end{lemma}
\begin{proof}
Note that conditioned on $\bm{X}$, the distribution of  $\bm{v} = (\bm{I}- \bm{X}(\bm{X}^\top \bm{X}+ \lambda \gamma \bm{I})^{-1} \bm{X}^{\top}) \bm{\epsilon}$  is multivariate Gaussian with mean zero and covariance matrix ${ \sigma_{\e}^2}(\bm{I}- \bm{X}(\bm{X}^\top \bm{X}+ \lambda \gamma \bm{I})^{-1} \bm{X}^{\top})^2$. We have
 \begin{eqnarray}
 (\bm{I}- \bm{X}(\bm{X}^\top \bm{X}+ \lambda \gamma\bm{I})^{-1} \bm{X}^{\top})^2 = \bm{I} -  \bm{X}(\bm{X}^\top \bm{X}+ \lambda \gamma \bm{I})^{-1} \bm{X}^{\top}- \lambda \gamma  \bm{X}(\bm{X}^\top \bm{X}+ \lambda \gamma \bm{I})^{-2} \bm{X}^{\top}.
 \end{eqnarray}
We define $\sigma_i^2(\bm{X})= \left(1 - \bm{x}_i^\top (\bm{X}^\top \bm{X}+ \lambda \gamma\bm{I})^{-1}\bm{x}_i - \lambda \gamma \bm{x}_i^\top (\bm{X}^\top \bm{X}+ \lambda \gamma \bm{I})^{-2}\bm{x}_i\right) \sigma_\epsilon^2$. Clearly $\sigma_i^2 (\bm{X}) \leq \sigma_\epsilon^2$, hence,
\begin{eqnarray}
\mathbb{E} ( |v_i|^k  \ | \ \bm{X})  \leq \sigma_i^k(X) (k-1)!! \leq \sigma_\epsilon^k (k-1)!!, 
\end{eqnarray}
where the first inequality is due to Lemma \ref{lem:momenet:Gauss}. Hence, again by the law of iterated expectation, we have 
\begin{eqnarray*}
\mathbb{E} ( |v_i|^k )  \leq \sigma_\epsilon^k (k-1)!!. 
\end{eqnarray*}
\end{proof}

\begin{lemma} \label{lem:l3}
Under the assumptions of Example \ref{ex:squareLoss} we have
\begin{eqnarray}
\lefteqn{\mathbb{E} |\bm{x}_i^\top (\bm{X}^\top \bm{X}+ \frac{\lambda \gamma}{2} \bm{I})^{-1} \bm{\rd}_{0.5}(\bl)|^k} \nonumber \\
&  \leq& \!\!\!\!\! 2^{2k-\frac{3}{2}} 1.5^{\frac{k}{2}}  \left( \frac{1}{\lambda^2 \gamma} \left(1 + \frac{\alpha (1- \gamma)}{ 2\gamma} \right) \right)^k  \left(\rho \frac{ \bm{\beta}^\top \bm{\beta}}{p}+ \sigma_\epsilon^2\right)^{\frac{k}{2}}  \sqrt{  (2k)!! (1 + \Big(\frac{1.5 c}{\sqrt{2 \lambda \gamma}} \Big)^{2k})} +  \zeta^{\frac{k}{2}} \Big).
\end{eqnarray}
\end{lemma}
\begin{proof}
Since $f_{\slash i} (\bli) \leq f_{\slash i} (\bm{0})$, we have
\begin{equation}\label{eq:upperbli}
{2}\lambda \gamma \|\bli\|_2^2 \leq \|\bm{y}_{\slash i}\|_2^2. 
\end{equation}
Furthermore, due to $\rdd_{0.5} (\beta) \leq  \gamma + \frac{\alpha  (1- \gamma)}{2}$,  $\rd_{0.5}(0)=0$, and  \eqref{eq:upperbli}, we have
\begin{equation}\label{eq:rdBound1}
\| \bm{\rd}_{0.5} (\bli)\|_2^2 \leq  \left( \gamma + \frac{\alpha  (1- \gamma)}{2}\right) \|\bli\|_2^2 \leq \left(\frac{1}{{2\lambda}} + \frac{\alpha (1- \gamma)}{{4\lambda} \gamma}\right) \|\bm{y}_{\slash i}\|_2^2.
\end{equation}
The first order optimality condition yields
\[
\bm{X}^\top \bm{X} (\bli - \bl) + \lambda \bm{\rd} (\bli) - \lambda \bm{\rd} (\bl) = - \bm{x}_i (y_i - \bm{x}_i^\top \bli).  
\]
Since the minimum eigenvalue of the Hessian of $\bm{r}(\bm{\beta})$ is $2 \gamma$, {therefore the minimum eigenvalue of $\bm{X}^\top \bm{X} + \lambda \diag[\bm{\rdd}(\bm{\beta})]$ (for all $\bm{\beta}$) is greater than $2\lambda \gamma$,} leading to  
\[
\|\bli- \bl\|_2 \leq \frac{|y_i -  \bm{x}_i^\top \bli|}{{ 2 }\lambda \gamma} \|\bm{x}_i\|_2. 
\]
This together with $\rdd_{0.5} (\beta) \leq  \gamma + \frac{\alpha  (1- \gamma)}{2}$ yields 
\[
\| \bm{\rd}_{0.5}(\bli) - \bm{\rd}_{0.5} (\bl)\|_2 \leq \left( \gamma+ \frac{\alpha  (1- \gamma)}{2}\right) \|\bli - \bl\|_2 \leq \left(\frac{1}{{2\lambda}} + \frac{\alpha (1- \gamma)}{{4\lambda} \gamma}\right) |y_i -  \bm{x}_i^\top \bli| \|\bm{x}_i\|_2.
\]
Define $\bm{D}_i = (\XI^\top \XI + \lambda \gamma \bm{I})^{-1}$. According to the matrix inversion lemma we have 
\begin{eqnarray}\label{eq:secondmainterm1}
\bm{x}_i^\top (\bm{X}^\top \bm{X}+ \lambda \gamma \bm{I})^{-1} \bm{\rd}_{0.5}(\bl) = \bm{x}_i^\top \bm{D}_i \bm{\rd}_{0.5} (\bl) - \frac{\bm{x}^\top_i \bm{D}_i \bm{x}_i \bm{x}_i^\top \bm{D}_i \bm{\rd}{}_{0.5} (\bl)  }{1+ \bm{x}_i^\top \bm{D}_i \bm{x}_i} = \frac{ \bm{x}_i^\top \bm{D}_i \bm{\rd}_{0.5} (\bl)}{1+ \bm{x}_i^\top \bm{D}_i \bm{x}_i}. 
\end{eqnarray}
Furthermore, we have
\begin{equation}\label{eq:x_iD_rdbreak}
 |\bm{x}_i^\top \bm{D}_i \bm{\rd}_{0.5} (\bl)| \leq |\bm{x}_i^\top \bm{D}_i \bm{\rd}_{0.5}(\bli)|+ |\bm{x}_i^\top \bm{D}_i (\bm{\rd}_{0.5}(\bl)- \bm{\rd}_{0.5}(\bli))|.
\end{equation}
Note that for two  random variables $a$ and $b$ we have
\[
\mathbb{E} (a+b)^k \leq 2^{k-1} \mathbb{E} (a^k + b^k).
\] 
Hence, 
\begin{eqnarray}
 \mathbb{E} (|\bm{x}_i^\top \bm{D}_i \bm{\rd}_{0.5} (\bl)|)^k \leq 2^{k-1} \left(\mathbb{E} |\bm{x}_i^\top \bm{D}_i \bm{\rd}_{0.5}(\bli)|^k+ \mathbb{E}|\bm{x}_i^\top \bm{D}_i (\bm{\rd}_{0.5}(\bl)- \bm{\rd}_{0.5}(\bli))|^k \right).
\end{eqnarray}
First note that, since the maximum eigenvalue of $\bm{D}_i $ is $\lambda \gamma$ we have
\begin{eqnarray}\label{eq:DiTwoDifferenceUpper1}
\lefteqn{ |\bm{x}_i^\top \bm{D}_i (\bm{\rd}_{0.5}(\bl)- \bm{\rd}_{0.5}(\bli))|} \nonumber \\&\leq& \frac{1}{\lambda \gamma} \|\bm{x}_i \|_2 \|\bm{\rd}_{0.5}(\bl)- \bm{\rd}_{0.5}(\bli)\|_2 
 \leq  \frac{1}{2 \lambda^{ 2} \gamma} \|\bm{x}_i \|^2_2 \left(1 + \frac{\alpha (1- \gamma)}{2 \gamma}\right)  |y_i -  \bm{x}_i^\top \bli| \nonumber \\
 &\leq& \frac{1}{2\lambda^2 \gamma} \left(1 + \frac{\alpha (1- \gamma)}{ 2\gamma} \right)   \|\bm{x}_i \|^2_2 (|y_i| + |\bm{x}_i^\top \bli| ).
\end{eqnarray}
Hence,
\begin{eqnarray}
\lefteqn{\mathbb{E} (|\bm{x}_i^\top \bm{D}_i (\bm{\rd}_{0.5}(\bl)- \bm{\rd}_{0.5}(\bli))|)^k \leq  \left( \frac{1}{\lambda^2 \gamma} \left(1 + \frac{\alpha (1- \gamma)}{ 2\gamma} \right) \right)^k \sqrt{ \mathbb{E} ( \|\bm{x}_i \|_2)^{2k}  \mathbb{E}(|y_i| + |\bm{x}_i^\top \bli| )^{2k}}} \nonumber \\
&\leq&  \left( \frac{1}{2\lambda^2 \gamma} \left(1 + \frac{\alpha (1- \gamma)}{ 2\gamma} \right) \right)^k 2^{(2k-1)/2} \sqrt{ \mathbb{E} ( \|\bm{x}_i \|_2)^{2k}  ( \mathbb{E}|y_i|^{2k} + \mathbb{E}|\bm{x}_i^\top \bli | ^{2k})} \hspace{3cm}
\end{eqnarray}

Furthermore, we have
\begin{enumerate}
\item According to Lemma \ref{lem:momenet:chi}, $\mathbb{E} \|\bm{x}_i \|_2^\ell = \frac{p(p+2)\ldots (p+\ell-2)}{p^\frac{\ell}{2}} \leq  \left(1+ \frac{\ell-2}{p} \right)^{\frac{\ell}{2}} \leq 1.5^{\frac{\ell}{2}}$, where the last inequality is according to the assumption $p>2 (\ell-2)$.  
\item Note that $y_i \sim N(0, \bm{\beta}^\top \bm{\Sigma} \bm{\beta} + \sigma_\epsilon^2)$. Furthermore, $\bm{\beta}^\top \bm{\Sigma} \bm{\beta} + \sigma_\epsilon^2 \leq \rho \frac{ \bm{\beta}^\top \bm{\beta}}{p}+ \sigma_\epsilon^2$. Hence, using the the moments of Gaussian (see Lemma \ref{lem:momenet:Gauss}), we have  
\begin{equation}
\mathbb{E} |y_i|^\ell  \leq  \left(\rho \frac{ \bm{\beta}^\top \bm{\beta}}{p}+ \sigma_\epsilon^2\right)^{\ell/2} \ell!!.
\end{equation}

\item Given  $\XI, \bm{y}_{\slash i}$, the distribution of $\bm{x}_i^\top \bli$ is $N(0, \bli^\top \bm{\Sigma} \bli)$. Furthermore, $\bli^\top \bm{\Sigma} \bli \leq \frac{c \bli^\top \bli}{n} \leq \frac{c \|\bm{y}_{\slash i}\|_2^2}{2 n\lambda \gamma}$, where the last inequality is due to \eqref{eq:upperbli}. Hence, we have
\begin{eqnarray}
 \mathbb{E} (| \bm{x}_i^\top \bli |^\ell \ | \ \XI, \bm{y}_{\slash i} ) \leq \left(\frac{c \|\bm{y}_{\slash i}\|_2^2}{2 n\lambda \gamma} \right)^{\ell/2} \ell!!.  
 \end{eqnarray}
Since $y_i \overset{i.i.d.}{\sim} N(0, \bm{\beta}^\top \bm{\Sigma} \bm{\beta} + \sigma_\epsilon^2)$, and $\bm{\beta}^\top  \bm{\Sigma} \bm{\beta} + \sigma_\epsilon^2 \leq \frac{\rho \bm{\beta}^\top  \bm{\beta}}{p} + \sigma_\epsilon^2$,  we have
\begin{eqnarray}\label{boundx_ibli}
 \mathbb{E} (| \bm{x}_i^\top \bli |^\ell  ) &\leq& \left(\frac{c^\ell \mathbb{E} ( \|\bm{y}_{\slash i}\|_2^\ell )}{ (2 n\lambda \gamma)^{\ell/2}} \right) \ell!! \leq \frac{c^\ell \left(\frac{\rho \| \bm{\beta}^\top  \bm{\beta}  \|_2^2}{p} + \sigma_\epsilon^2 \right)^\ell} {(2 \gamma \lambda)^\frac{\ell}{2}}  \ell!! \frac{n(n+2) \ldots (n+\ell-2)}{n^{\ell/2}} \nonumber \\
 &\leq& \frac{c^\ell \left(\frac{\rho \| \bm{\beta}^\top  \bm{\beta}  \|_2^2}{p} + \sigma_\epsilon^2 \right)^\ell} {(2 \gamma \lambda)^\frac{\ell}{2}}   1.5^\ell \ell!!,
 \end{eqnarray}
where for the last inequality we assumed that $n> 2 \ell$. 

\end{enumerate}

Finally, we compute an upper bound on $ |\bm{x}_i^\top \bm{D}_i \bm{\rd}_{0.5}(\bli)|$. Since $\bm{x}_i$ is independent of $\bm{y}_{\slash i}$ and $\XI$, we conclude that given  $\XI$ and $\bm{y}_{\slash i}$, $\bm{x}_i^\top \bm{D}_i \bm{\rd}_{0.5}(\bli)$ is a Gaussian random variable with mean zero and variance 
$$\|\bm{\Sigma}^{1/2} \bm{D}_i \bm{\rd}_{0.5}(\bli)\|^2_2 \leq \frac{4 \rho_{\max}}{\lambda^2 \gamma^2}\| \bm{\rd}_{0.5}(\bli)\|^2_2 \leq  \frac{{2} \rho_{\max}}{\lambda^{3} \gamma^2} \left(1 + \frac{\alpha (1- \gamma)}{2\gamma}\right) \|\bm{y}_{\slash i}\|_2^2 = \frac{\zeta \|\bm{y}_{\slash i}\|_2^2}{n},$$
where $\zeta = \frac{2 c}{\lambda^3 \gamma^2} \left(1 + \frac{\alpha (1- \gamma)}{2\gamma}\right)$, and the second inequality is due to \eqref{eq:rdBound1}. Hence,
\begin{eqnarray*}
\mathbb{E} \|\bm{\Sigma}^{1/2} \bm{D}_i \bm{\rd}_{0.5}(\bli)\|^{\ell}_2 \leq \zeta^{\ell/2} \frac{n (n+2) \ldots (n+ \frac{\ell}{2} -2)}{n^{\ell/2}} \leq (1.5 \zeta)^{\ell/2}. 
\end{eqnarray*}

\end{proof}

%---------------------%-----------------------------%
%---------------------%-----------------------------%

%---------------------%-----------------------------%
%---------------------%-----------------------------%

\section{Proof of Corollary \ref{ex:PoissonLoss}}\label{sec:proof:ex:PoissonLoss}

The goal of this section is to use Theorem \ref{th:mse2} to prove corollary \ref{ex:PoissonLoss}. Hence, we have to confirm that Assumptions \ref{Aprime1}, \ref{Aprime2}, and \ref{Aprime3} hold, and that $\E \var[ \phi(y_o, \bm{x}_o^\top \blone) \mid D_{/1}]$ is bounded. Similar to what we did at the beginning of Section \ref{ssec:proof:ExampleElasticNet}, it is straightforward to check the validity of Assumptions \ref{Aprime1} and \ref{Aprime3}. Hence, we only focus on proving Assumption \ref{Aprime2} and finding an upper bound for $\E \var[ \phi(y_o, \bm{x}_o^\top \blone) \mid D_{/1}]$.  

  Regarding Assumption \ref{Aprime2}, we first prove that under the assumptions of this corollary, there exists a fixed number $\tilde{c}_0$, such that $\E (\ld(y_i \mid \xv_i^\top \bl))^8 \leq \tilde{c}_0$
  and $\E (\ld(y_0 \mid \xv_0^\top \bli))^8 \leq \tilde{c}_0$.
Since $\ell(y \mid z)=f(z)-y\log f(z)$, we have
  \begin{equation*}
    \ld(y_i \mid \bm{x}_i^\top \bl) = f'(\bm{x}_i^\top \bl) - y_i f'(\bm{x}_i^\top \bl) / f(\bm{x}_i^\top \bl),
  \end{equation*}
  where $f'(z) = 1 / (1 + e^{-z})$. We have that, for all $z \in \mathbb{R}$, $f'(z) \leq 1$ and
  $0 \leq f'(z) / f(z) \leq 1$, from which we deduce that:
  \begin{equation}\label{eq:upperPoissonLoss}
    \abs{\ld(y_i \mid \bm{x}_i^\top \bl)} \leq 1 + y_i.
  \end{equation}
  In particular, we have that:
  \begin{eqnarray*}
    \E \abs{\ld(y_i \mid \bm{x}_i^\top \bl)}^8
    &\leq& \E (1 + y_i)^8 \nonumber \\
    &\leq& \E e^{8 y_i} 
    = \E \E [e^{8 y_i} \mid \bm{x}_i^\top \bm\beta^*] \\
    &\overset{(a)}{=}& \E \exp\{(e^8 - 1) \bm{x}_i^\top \bm\beta^*\}  \nonumber \\
    &\overset{(b)}{\leq}& \exp\bigl\{ \frac{\rho}{2p}\norm{\bm\beta^*}_2^2 (e^8 - 1)^2 \bigr\}  \nonumber \\
    &=& \exp\bigl\{ \frac{(e^8 - 1)^2}{2} \frac{\rho}{p} \norm{\bm\beta^*}_2^2 \bigr\}, \\
    & \leq& \exp\bigl\{ \frac{(e^8 - 1)^2}{2} \rho b  \bigr\}.
  \end{eqnarray*}
  To obtain equality (a) we have used the moment generating function of the Poisson distribution with $y_i \sim Poisson(f(\bm{x}_i^\top \bm{\beta}^*))$.
  To obtain inequality (b) we have used the moment generating function of a Gaussian distribution and the fact that $\mathbb{E} (\bm{x}_i^\top \bm\beta^*)^2 \leq \frac{\rho}{p} \norm{\bm\beta^*}_2^2$. Given that the upper bound we derived in \eqref{eq:upperPoissonLoss} for the derivative of the loss function does not depend on the second input argument of the loss, that is $\bm{x}_i^\top \bl$, the proof that Poisson loss satisfies the other conditions of Assumption \ref{Aprime2} for $\phi(y, z) = \ell(y \mid z)$ will be exactly similar and hence is skipped.
  In particular, we have verified the conditions of Assumption \ref{Aprime2} for any convex regularizer.

Now we turn our attention to bounding $\E \var[ \ell(y_o \mid \bm{x}_o^\top \blone) \mid D_{/1}]$. First note that
\begin{eqnarray}
\var[ \ell(y_o| \bm{x}_o^\top \blone) \mid D_{/1}] \leq \E [ \ell^2(y_o| \bm{x}_o^\top \blone) \mid D_{/1}].
\end{eqnarray}
Furthermore, from the mean value theorem we have:
\[
\ell(y_o| \bm{x}_o^\top \blone) = \ell(y_o \mid \bm{x}_o^\top \bm\beta^* ) + \ld(y_o| \tilde{z}) (\bm{x}_o^\top \blone - \bm{x}_o^\top \bm\beta^*),
\]
Hence, we have:
\begin{eqnarray}\label{eq:overallTerm}
\ell^2(y_o \mid \bm{x}_o^\top \blone)  \leq 2 \ell^2(y_o \mid \bm{x}_o^\top \bm\beta^*) + 2(1+y_o^2)(\bm{x}_o^\top \blone - \bm{x}_o^\top \bm\beta^*)^2.  
\end{eqnarray}
To complete the proof we have to show that both $\E \ell^2(y_o, \bm{x}_o^\top \bm\beta^*)$ and $\E (1+y_i^2)(\bm{x}_o^\top \blone - \bm{x}_o^\top \bm\beta^*)^2$ are bounded. First note that, using $\ell(y \mid z)=f(z)-y\log f(z)$ and, for any $a, b \in \R$, $(a + b)^2 \leq 2a^2 + 2b^2$, yields
\begin{eqnarray}
\ell^2(y_o \mid \bm{x}_o^\top \bm\beta^*) \leq 2 f^2(\bm{x}_o^\top \bm\beta^*) +2 y_o^2 \log^2 f(\bm{x}_o^\top \bm\beta^*).
\end{eqnarray}
Hence,
\begin{eqnarray}
\E \ell^2(y_o \mid \bm{x}_o^\top \bm\beta^*) \leq 2 \E f^2(\bm{x}_o^\top \bm\beta^*) +2\E (f(\bm{x}_o^\top \bm\beta^*) +  f^2(\bm{x}_o^\top \bm\beta^*)) \log^2 f(\bm{x}_o^\top \bm\beta^*).
\end{eqnarray}
The following facts will help us bound these terms:
\begin{eqnarray}
f(\bm{x}_o^\top \bm\beta^*) &\geq& 0 \nonumber \\
f(\bm{x}_o^\top \bm\beta^*) &\leq& 1+ |\bm{x}_o^\top \bm\beta^*|, \nonumber \\
\end{eqnarray}
On the other hand, it is straightforward to check that for any $\gamma>0$ we have
\begin{eqnarray}
\gamma \log^2 \gamma &\leq& 1+ \gamma^2, \nonumber \\
\gamma^2 \log^2 \gamma &\leq& 1+ \gamma^3. 
\end{eqnarray}
By combining these equations we obtain:
\begin{eqnarray}\label{eq:boundellO}
\E \ell^2(y_o \mid \bm{x}_o^\top \bm\beta^*) &\leq& 2 \E f^2(\bm{x}_o^\top \bm\beta^*) +2\E (1+f^2(\bm{x}_o^\top \bm\beta^*))  + 2\E (1+f^3(\bm{x}_o^\top \bm\beta^*)) \nonumber \\
&\leq& 4+ 4 \E f^2(\bm{x}_o^\top \bm\beta^*) + 2 \E f^3(\bm{x}_o^\top \bm\beta^*) \nonumber \\
&\leq& 4+ 4 \E (1+ \abs{\bm{x}_o^\top \bm\beta^*})^2 + 2 \E(1+ \abs{\bm{x}_o^\top \bm\beta^*})^3. 
\end{eqnarray}
Note that $\bm{x}_o^\top \bm\beta^*$ is a Gaussian random variable with mean zero and variance $(\bm \beta^*)^\top \bm{\Sigma} \bm\beta^* \leq \rho b$. Hence, $\E \ell^2(y_o, \bm{x}_o^\top \bm\beta^*)$ is bounded by a constant. 

For the second term in \eqref{eq:overallTerm} we have
\begin{eqnarray}
\mathbb{E}  (1+y_o^2)(\bm{x}_o^\top \blone - \bm{x}_o^\top \bm\beta^*)^2  &=& \E \left(1+ f(\bm{x}_o^\top \bm\beta^*) + f^2(\bm{x}_o^\top \bm\beta^*) \right )  (\bm{x}_o^\top \blone - \bm{x}_o^\top \bm\beta^*)^2 \nonumber \\
&\leq & \E \left(1+ (1+ |\bm{x}_o^\top \bm\beta^*{|} ) + (1+|\bm{x}_o^\top \bm\beta^*|^2) \right) (\bm{x}_o^\top \blone - \bm{x}_o^\top \bm\beta^*)^2. 
\end{eqnarray}
Note that in order to show that this term is bounded from above by a constant, we only need to show that terms of the form:
\[
\E |\bm{x}_o^\top \blone|^{k_1} |\bm{x}_o^\top \bm\beta^*|^{k_2} \leq  (\E |\bm{x}_o^\top \blone|^{2k_1} \E|\bm{x}_o^\top \bm\beta^*|^{2k_2})^{1/2}
\]
are bounded for $k_1 \leq 2$ and $k_1 + k_2 \leq 4$. As previously, we note that $\bm{x}_0^\top \bm{\beta}^*$ is a Gaussian random variable with variance $\bm\beta^{*\top} \bm\Sigma \bm\beta^* \leq \frac{\rho}{p} \norm{\bm \beta^*}_2^2 \leq \rho b$, and hence $(\E|\bm{x}_o^\top \bm\beta^*|^{2k_2})^{1/2}$ is bounded. Hence, the only remaining step is to prove the boundedness of $\E |\bm{x}_o^\top \blone|^{2k_1}$, where $k_1$ is at most 2. Note that conditioned on $D_{/1}$ the random variable $\bm{x}_o^\top \blone$ is Gaussian with the variance that is bounded by $\frac{\rho}{p} \blone^\top \blone$.
Hence, using Lemma \ref{lem:momenet:Gauss} we have
\[
\E |\bm{x}_o^\top \blone|^{2k_1} \leq (2k_1-1)!!\mathbb{E} \left(\frac{\rho}{p} \blone^\top \blone \right)^{k_1}. 
\]
The definition of $\blone$ (and comparing it with $\bm{\beta}^*$) yields
\[
\sum_{j \neq i} \ell(y_j \mid \bm{x}_j^\top \blone) + \lambda r(\blone) \leq \sum_{j \neq i} \ell(y_j \mid \bm{x}_j^\top  \bm\beta^*) + \lambda r( \bm\beta^*), 
\]
The $\gamma$-strong convexity of the smoothed  elastic-net regularizer $r$, and the fact that $\ell \geq 0$, leads to 
\[
\lambda \gamma \|\blone\|_2^2 \leq \sum_{j \neq i} \ell(y_j \mid \bm{x}_j^\top  \bm\beta^*) + \lambda r( \bm\beta^*). 
\]
Since $k_1 \leq 2$, we only prove that $\E |\bm{x}_o^\top \blone|^{4}$ is bounded.
Toward this goal we have:
\begin{eqnarray}
\E \Bigl( \frac{\lambda \gamma}{p} \|\blone\|_2^2 \Bigr)^2 &\leq& \frac{1}{p^2}  \E \Bigl(\sum_{j \neq i} \ell(y_j \mid \bm{x}_j^\top  \bm\beta^*) + \lambda r( \bm\beta^*)\Bigr)^2 \nonumber \\
&\leq& \frac{2}{p^2}  \E (\sum_{j \neq i} \ell(y_j \mid \bm{x}_j^\top  \bm\beta^*) )^2 + \E ( \lambda r( \bm\beta^*))^2 \nonumber \\
&\leq& \frac{2n (n-1)}{p^2}  \E \ell^2(y_1 \mid \bm{x}_1^\top  \bm\beta^*) + \frac{  \lambda^2 r^2( \bm\beta^*)}{p^2} \nonumber \\
&\leq& 2 \delta^2  \E \ell^2(y_1 \mid \bm{x}_1^\top  \bm\beta^*) + \frac{  \lambda^2 r^2( \bm\beta^*)}{p^2}.
\end{eqnarray}
Hence, we have to prove that $\E \ell^2(y_1 \mid \bm{x}_1^\top  \bm\beta^*)$ and $\frac{  \lambda^2 r^2( \bm\beta^*)}{p^2}$ are bounded. 
First note we proved in \eqref{eq:boundellO} that:
\begin{eqnarray}
 \ell^2(y_o \mid \bm{x}_o^\top \bm\beta^*) \leq 4+ 4 (1+ \abs{\bm{x}_o^\top \bm\beta^*})^2 + 2 (1+ \abs{\bm{x}_o^\top \bm\beta^*})^3. 
\end{eqnarray}
Note that $\bm{x}_0^\top \bm{\beta}^*$ is a Gaussian random variable with variance $\bm\beta^{*\top} \bm\Sigma \bm\beta^* \leq \frac{\rho}{p} \norm{\beta^*}_2^2 \leq \rho b$, and hence $\E \ell^2(y_o \mid \bm{x}_o^\top \bm\beta^*)$ is bounded. On the other hand,
\begin{eqnarray}
r(\bm{\beta}^*) =  \gamma (\bm{\beta}^*)^\top \bm{\beta}^* + (1-\gamma) \sum_{i=1}^p r^{\alpha}(\beta^*_i). 
\end{eqnarray}
It is straightforward to prove that $\rd^\alpha(z) = \frac{\rm{e}^{\alpha z} -\rm{e}^{-\alpha z} }{\rm{e}^{\alpha z} +\rm{e}^{-\alpha z}+1}<1$. Hence,
\begin{eqnarray}\label{eq:almostlast_cor4}
r(\bm{\beta}^*) =  \gamma (\bm{\beta}^*)^\top \bm{\beta}^* + (1-\gamma) \sum_{i=1}^p r^{\alpha}(\beta^*_i) <   \gamma (\bm{\beta}^*)^\top \bm{\beta}^* + (1-\gamma)\sum_{i=1}^p(\frac{2 \log 2}{\alpha} + |\beta^*_i|),
\end{eqnarray}
where to obtain the last inequality we used the mean value theorem  
\[
r^\alpha(|z|)  = r\alpha(0) + \rd^\alpha(\tilde z) |z|, 
\]
where $\tilde{z} \in (0, |z|)$, and the facts that $\rd^\alpha(\tilde{z}) \leq 1$ and $r^\alpha (0) = \frac{2 \log 2}{\alpha}$. Using \eqref{eq:almostlast_cor4} we obtain:
\begin{eqnarray}
\frac{1}{p} r(\bm{\beta}^*) &\leq& \frac{\gamma (\bm{\beta}^*)^\top \bm{\beta}^*}{p} + \frac{(1-\gamma) 2 \log 2}{\alpha} + \frac{1-\gamma}{p} \sum_{i=1}^p |\beta^*_i| \nonumber \\
&\leq& \frac{\gamma (\bm{\beta}^*)^\top \bm{\beta}^*}{p} + \frac{(1-\gamma) 2 \log 2}{\alpha} + (1-\gamma)\sqrt{\frac{{\sum_{i=1}^p |\beta^*_i|^2}}{{p}}} \nonumber 
\\ &\leq& \gamma b + \frac{(1-\gamma) 2 \log 2}{\alpha} + (1-\gamma)\sqrt{b}. 
\end{eqnarray}

%---------------------%-----------------------------%
%---------------------%-----------------------------%

\section{Proof of Corollary \ref{ex:negbino}}\label{sec:proof:ex:negbino}

Similar to the proofs of Corollaries \ref{ex:PoissonLoss}, \ref{ex:squareLoss}, we would like to use Theorem \ref{th:mse2} to prove our claim. Toward this goal, We have to prove that Assumptions \ref{Aprime1}, \ref{Aprime2}, and \ref{Aprime3} hold. Furthermore, we have to obtain an upper bound for the constant $\tilde{C}_v$, which in turn requires us to bound $\E \var[ \phi(y_o, \bm{x}_o^\top \blone) \mid D_{/1}]$. Again, the proofs of Assumptions \ref{Aprime1} and \ref{Aprime3} are exactly the same as we presented in the last two sections. Hence, we only focus on Assumption \ref{Aprime2} and $\E \var[ \phi(y_o, \bm{x}_o^\top \blone) \mid D_{/1}]$. We would like to prove that the conditions of Assumption \ref{Aprime2} are satisfied with $\tilde{c}_0 = \tilde{c}_1 = 2^8 ( \kappa + \alpha^{-8})$.

    It we compute the derivative of the log-likelihood, we will obtain
    \begin{equation}\label{eq:neg-binomial-gradient}
        \abs{\ld(y \mid z)}
        = \abs*{-y + (y + \alpha^{-1})\frac{\alpha {\rm e}^z}{1 + \alpha {\rm e}^z}}
        \leq y + \alpha^{-1}.
    \end{equation}
    We thus deduce that:
    \begin{equation*}
        \E \abs{\ld(y_1 \mid \xv_1^\top \bl)}^8 \leq \E (y + \alpha^{-1})^8 \leq 2^8 ( \kappa + \alpha^{-8}).
    \end{equation*}
    As the bound \eqref{eq:neg-binomial-gradient} is free of $z$, the same argument above applies to the other requirements in Assumption \ref{Aprime2}. 
    
    Now we turn our attention to the calculation of $\E \var[ \ell(y_o, \bm{x}_o^\top \blone) \mid D_{/1}]$. Note that
    \[
    \E \var[ \ell(y_o, \bm{x}_o^\top \blone) \mid D_{/1}] \leq \E \ell^2(y_o, \bm{x}_o^\top \blone). 
    \]
    Note that by removing the constant from the log-likelihood we obtain
    \begin{eqnarray}
    |\ell(y_o \mid \bm{x}_o^\top \blone)| &=& |(y_o + \alpha^{-1})\log (1+ \alpha {\rm e}^{\bm{x}_o^\top \blone}) - y_o (\bm{x}_o^\top \blone) |\leq  | y_o + \alpha^{-1}|(1+ |\log \alpha| +|\bm{x}_o^\top \blone|) + y_o |\bm{x}_o^\top \blone| \nonumber \\
    &\leq& 2y_o |\bm{x}_o^\top \blone|+ \alpha^{-1}(1+ |\log \alpha| +|\bm{x}_o^\top \blone|).
    \end{eqnarray}
    The rest of the proof is very similar to the proof that we presented for Corollary \ref{ex:PoissonLoss}. Hence, we skip it. 

\end{document}